\documentclass[10pt]{article} 
\usepackage[preprint]{tmlr}
\usepackage{mathtools}
\usepackage{booktabs}       

\usepackage{amsmath,amsfonts,bm}

\def\tblref#1{table~\ref{#1}}

\def\figref#1{figure~\ref{#1}}

\def\secref#1{section~\ref{#1}}

\def\eqref#1{equation~\ref{#1}}

\def\1{\bm{1}}

\def\rvz{{\mathbf{z}}}

\def\vzero{{\bm{0}}}

\def\vlambda{{\bm{\lambda}}}
\def\vepsilon{{\bm{\varepsilon}}}
\def\va{{\bm{a}}}

\def\vc{{\bm{c}}}
\def\vd{{\bm{d}}}
\def\ve{{\bm{e}}}

\def\vg{{\bm{g}}}
\def\vh{{\bm{h}}}

\def\vq{{\bm{q}}}

\def\vs{{\bm{s}}}

\def\vu{{\bm{u}}}
\def\vv{{\bm{v}}}

\def\vx{{\bm{x}}}
\def\vy{{\bm{y}}}

\def\evepsilon{{\varepsilon}}

\def\evg{{g}}
\def\evh{{h}}

\def\evv{{v}}

\def\evx{{x}}

\def\mA{{\bm{A}}}
\def\mB{{\bm{B}}}
\def\mC{{\bm{C}}}
\def\mD{{\bm{D}}}
\def\mE{{\bm{E}}}

\def\mG{{\bm{G}}}

\def\mI{{\bm{I}}}

\def\mL{{\bm{L}}}

\def\mP{{\bm{P}}}
\def\mQ{{\bm{Q}}}

\def\mU{{\bm{U}}}
\def\mV{{\bm{V}}}

\def\mZ{{\bm{Z}}}

\DeclareMathAlphabet{\mathsfit}{\encodingdefault}{\sfdefault}{m}{sl}
\SetMathAlphabet{\mathsfit}{bold}{\encodingdefault}{\sfdefault}{bx}{n}

\def\sD{{\mathbb{D}}}

\def\sK{{\mathbb{K}}}

\def\sN{{\mathbb{N}}}

\def\sR{{\mathbb{R}}}
\def\sS{{\mathbb{S}}}

\newcommand{\E}{\mathbb{E}}

\usepackage{amsthm}
\usepackage{hyperref}
\usepackage{url}
\usepackage{graphicx}
\usepackage{subcaption}

\newtheorem{claim}{Claim}
\newtheorem{lemma}{Lemma}
\newtheorem{corollary}{Corollary}
\newtheorem{theorem}{Theorem}
\newtheorem{definition}{Definition}

\title{The spectral neuron}

\author{\name Alex Shoff \email alexander.shtoff@tii.ae \\
      \addr Technology Innovation Institute}

\def\month{MM}  
\def\year{YYYY} 
\def\openreview{\url{https://openreview.net/forum?id=XXXX}} 
\def\symd{\sS^d}
\DeclareMathOperator{\sym}{sym}

\begin{document}

\begin{flushright}
\begin{minipage}{0.55\textwidth}
\itshape
``There is a possibility that the human neurons do more compute than we think.''

\raggedleft
--- Ilya Sutskever, 2025
\end{minipage}
\end{flushright}

\maketitle

\begin{abstract}
As machine learned models increase in complexity and expressive power, features of simpler models, such as intrinsic coefficient transparency and control over the shape of the modeled function are lost. On the one edge of the spectrum we have simple linear models that possess coefficient transparency, but have a limited expressive power. On the other edge we have neural networks, that have expressive power that improves with scaling, but are mostly opaque.  In this work we develop the \emph{spectral neuron} concept: a scalar model given by $f(\vx)=\lambda_k\!\left(\mA_0+\sum_{i=1}^n x_i \mA_i\right)$, with learned real symmetric matrices \(\mA_0,\ldots,\mA_n\). The input enters the model through an affine matrix function, but the prediction is obtained by reading one of its eigenvalues. Thus, the model is nonlinear, but the source of nonlinearity is still mathematically explicit. This gives us a useful middle ground: the model can become more expressive as the matrix dimension grows, while retaining coefficient transparency through the learned matrices. For example, extremal eigenvalues yield convex or concave functions, semidefinite constraints on the coefficient matrices impose monotonicity, and the associated eigenspaces characterize local feature influence. We study coefficient transparency, feature-influence bounds, and shape-control properties of this model family, and then test whether it can be learned and scaled in practice. We develop a systematic study of this model family, bringing together spectral results from several mathematical literatures to characterize its expressivity, coefficient transparency, feature influence, and shape-control properties. Code available at \url{https://github.com/alexshtf/spectral_neuron_paper}.
\end{abstract}

\section{Introduction}

Many practical domains where tabular data is the primary modality benefit from models that are performant and possess coefficient transparency. We may need to explain to buyers of insurance policies why we assess the risk as we do, to explain to business stakeholders or regulators how our models make predictions and which features have the greatest influence on the predictions, perform feature selection, or control and understand how the model's predictions change when the input features change. Generalized linear models, despite their limited expressive power, are a convenient default - the coefficients tell almost the entire story.

As machine learned models increase in complexity, we lose some of these properties. For example, for a tree ensemble we can no longer quantify the effect of changing a feature on the outcome. The modeled function is discontinuous, and there may be jumps of arbitrary sizes. While this may be desirable in some cases to correctly model reality, this is not always the case. In the extreme case we have neural networks, that even with one hidden layer are mostly opaque in practice.

In this work we study a family of models of the form
\begin{equation}\label{eq:model_family}
f_k(\vx;\mA_{0:n}) = \lambda_k(\mathcal{A}(\vx)), \quad \mathcal{A}(\vx) = \mA_0 + \sum_{i=1}^n \evx_i \mA_i,
\end{equation}
where $\lambda_k$ is the $k$-th smallest eigenvalue. The learnable parameters are the symmetric matrices $\mA_0, \ldots, \mA_n \in \symd$, collectively referred to as $\mA_{0:n}$, and the model has one hyperparameter - the eigenvalue index $k$. We call such a model a \emph{spectral neuron}, due to its similarity to the classical ``neuron'' in machine learning - a composition of a nonlinear function onto a linear map. The idea is illustrated in \figref{fig:spectral_neuron_diagram}.

\begin{figure}[htbp]
    \centering
    \includegraphics[width=0.66\linewidth]{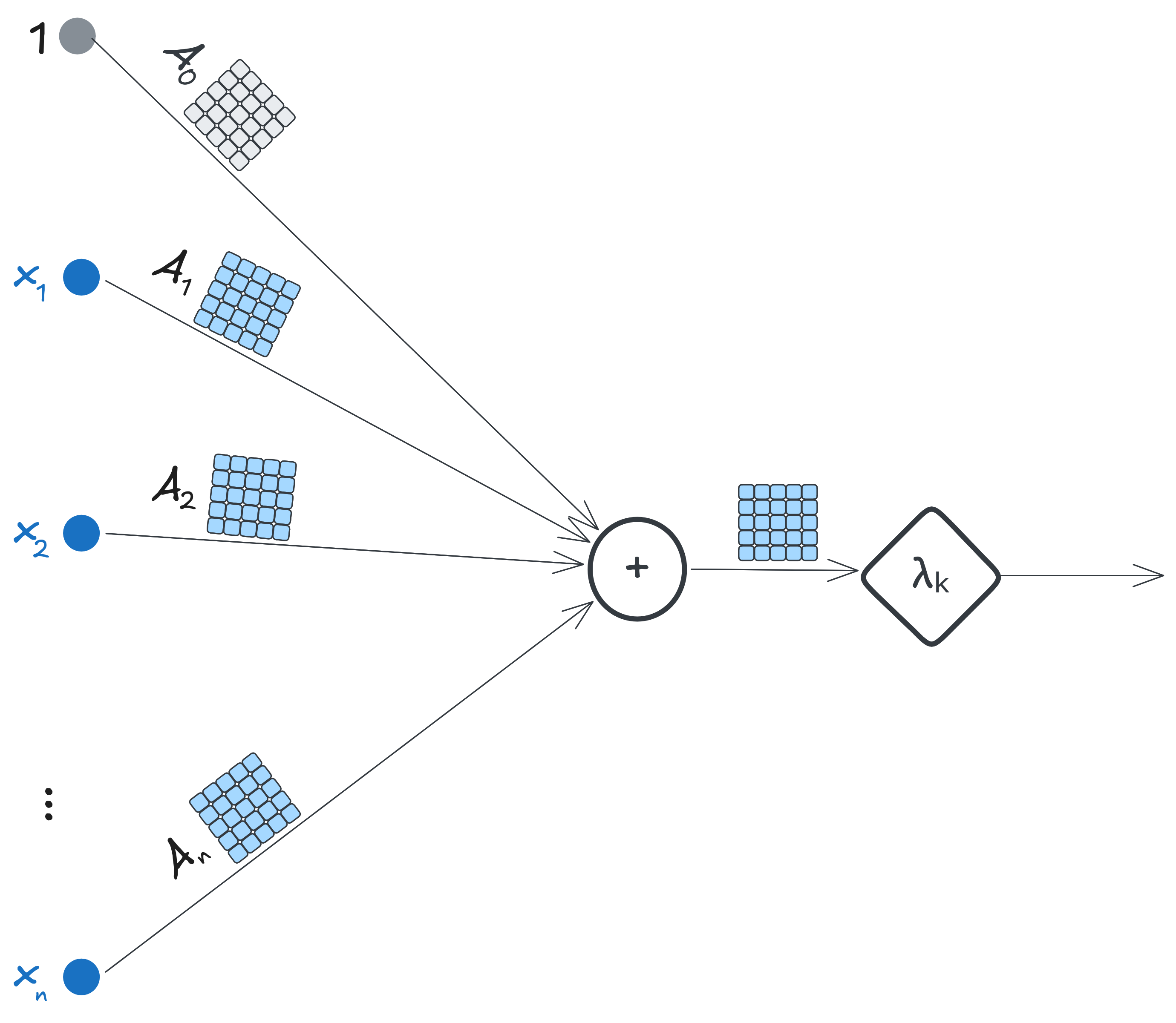}
    \caption{A diagram of the spectral neuron. The ``weights'' on the edges are matrices, rather than numbers. The ``non-linearity'' is an eigenvalue.}
    \label{fig:spectral_neuron_diagram}
\end{figure}

We find this model family worthy of systematic study because it combines
several properties that rarely occur together: (a) its expressive power grows
with the matrix dimension, providing a natural way to scale the model; (b) it
can be implemented and trained using standard eigensolvers, automatic
differentiation, and gradient-based optimization, without requiring a
specialized learning algorithm; (c) the learned matrices and their eigenspaces
provide coefficient transparency, supporting mathematically grounded analyses
of shape and feature influence; and (d) the coefficient matrices provide
easily computable global feature-influence bounds.

Many of the mathematical ingredients underlying these properties are
classical, and appear separately in matrix analysis, variational analysis,
perturbation theory, and semidefinite optimization. The main contribution of
this work is to bring them together into a unified treatment of the
\emph{spectral neuron} as a machine-learning primitive, and to study the
expressivity, shape control, coefficient transparency, feature influence, optimization, and
practical behavior of the resulting model family. Moreover, we show a good initialization strategy that works reasonably well in practice,  and conduct scaling experiments that corroborate our claims about the ability to scale the models.

\section{Related work}

The closest precedent, and the framework that already contains the core
construction studied here, is the parametric matrix model (PMM) framework of
\citet{cook2025parametric}. The authors establish universality and provide some physical-system interpretations. Our spectral neuron is just a special case of PMMs obtained by using an affine matrix pencil and a specific eigenvalue. Perhaps exactly because the PMM authors aimed to study a more general model and came from a different set of applications in mind, they may have overlooked that this special case combines shape control, global feature-influence bounds, and coefficient transparency while being able to improve with scaling, which is the main focus of this work.

Going the other direction, and restricting the affine matrix pencil further by assuming all matrices are simultaneously diagonalizable, and the model reduces to an order statistic of learned affine functions. The largest eigenvalue becomes a maxout unit
\citep{goodfellow2013maxout}, while internal eigenvalues become scalar
order-statistic or sorting units \citep{rennie2014deep,anil2019sorting}. Matrices which are not simultaneously diagonalizable allow us to gain additional expressive power over piecewise linear models.

Shape-constrained learning provides a further comparison. Monotonic and
input-convex networks, calibrated and deep lattices, and shape-constrained
additive models obtain guarantees through specialized architectural or
parameter restrictions \citep{sill1997monotonic,dugas2009functional,amir2017inputconvex,
gupta2016monotonic,you2017deeplattice,pya2015shape}. Despite their expressive power, deep neural networks lose transparency as they become deeper, while spectral neurons retain coefficient transparency as they scale. The mathematical structure remains the same - an eigenvalue of a symmetric matrix pencil.

Generalized additive models, models with selected pairwise interactions, and
neural additive models provide transparency by exposing low-dimensional feature
effects directly
\citep{hastie1986generalized,lou2013accurate,agarwal2021neural}. A spectral
neuron permits more general interactions and provides coefficient transparency
similar to that of linear models. In a linear model, each feature
has a scalar coefficient; in a spectral neuron, it has a coefficient matrix.
The spectral norm of this matrix provides a global feature-influence bound:
like the magnitude of a linear coefficient, it bounds the output change per
unit change in that feature over all inputs. At a particular input, the interaction between
the coefficient matrix and the selected eigenspace determines signed local
feature influence at simple eigenvalues and a local feature-influence bound at
repeated eigenvalues \citep{urruty1999clarkesubeigenvalue}. These quantities are
unchanged by a common orthogonal change of basis, and therefore belong to the
model itself rather than to an arbitrary matrix representation. Accumulating
the local responses along a path gives integrated-gradient attributions
\citep{sundararajan2017axiomatic}.

Neither local attribution nor global feature-influence bounds are unique to spectral
neurons. Gradients and integrated gradients are available for generic
differentiable models, while norm-based global guarantees are central to
Lipschitz-constrained networks \citep{anil2019sorting}. What distinguishes
spectral neurons is that both kinds of information follow directly from the
same coefficient matrices and remain just as simple as the model scales. A
larger spectral neuron can represent more complex functions, and become more
accurate, without becoming harder to reason about.

Finally, \emph{spectral} has a different meaning in classical machine learning.
Kernel principal component analysis, spectral clustering, and Laplacian
eigenmaps use eigenspaces of a dataset-level matrix to summarize, embed, or
partition observations
\citep{scholkopf1998nonlinear,ng2001spectral,belkin2003laplacian}. A spectral
neuron instead constructs a learned matrix for each input and uses one of its
eigenvalues as the prediction mechanism. The spectrum is part of the model
evaluated on each input, not a summary of the dataset.

\section{Notation}

The set of $d \times d$ symmetric matrices is denoted by $\symd$. The Euclidean norm of a vector $\vx$ is denoted by $\|\vx\|_2$, and similarily the spectral norm of a matrix $\mA$ is denoted by $\|\mA\|_2$. As the spectral theorem ensures that any $\mA \in \symd$ has $d$ real eigenvalues, we denote by $\lambda_i(\mA)$ the $i$-th \emph{smallest} eigenvalue of $\mA$. Note that this typically differs from many texts, that define $\lambda_i$ to be the $i$-th largest. For convenience, the smallest and largest eigenvalues are denoted by $\lambda_{\min}$ and $\lambda_{\max}$, respectively.

As a notational syntactic shortcut, we define $\mathcal{A}(\vx) \equiv \mA_0 + \sum_{i=1}^n \vx_i \mA_i$, where the matrices $\mA_0, \ldots, \mA_n$ are defined in the context where $\mathcal{A}(\vx)$ is used. Similarly, we define $\mathcal{B}(\vx) \equiv \mB_0 + \sum_{i=1}^n \mB_i$.

\section{Modeling power}
In this section we dive into continuity properties, representation power, and ways to reason about the model family that sheds light into its behavior. To gain intuition, we first inspect \figref{fig:1d_eigenvalues}, where a set of univariate functions $f_k$ are plotted, of the form 
\[
    f_k(x; \mA, \mB) = \lambda_k(\mA + x \mB),
\]
where $\mA$ and $\mB$ are $5 \times 5$ symmetric matrices with entries sampled from $[0, 1]$ uniformly at random.
\begin{figure}[htbp]
    \centering
    \begin{subfigure}[c]{\textwidth}
        \centering
        \includegraphics[width=.75\textwidth]{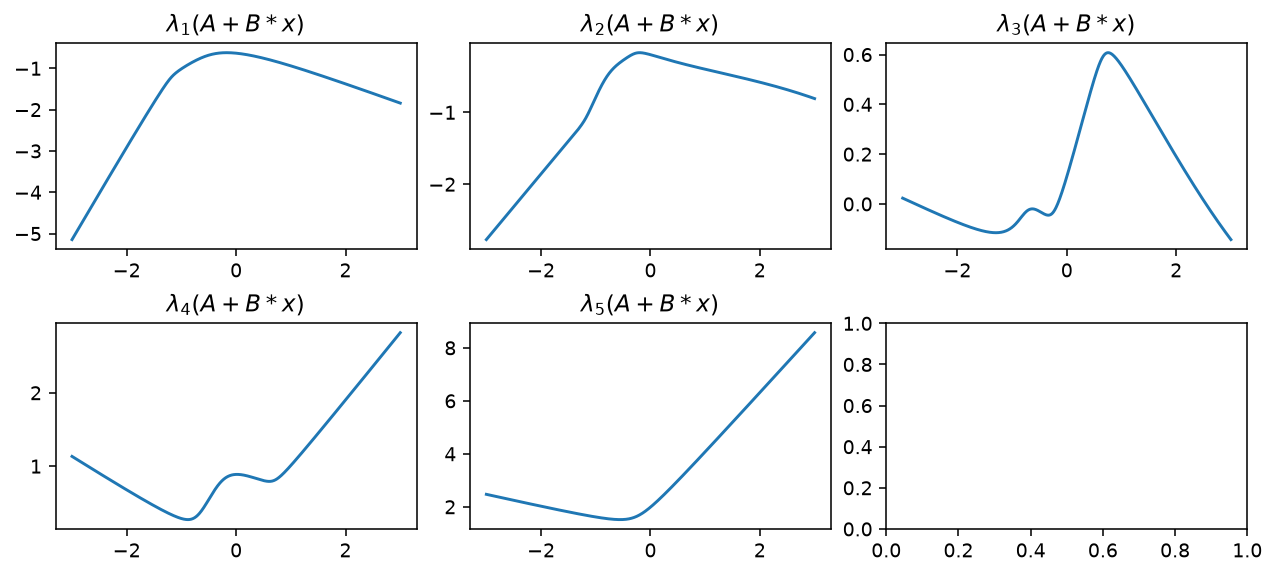}
        \caption{Gallery plot - each model in its own axis}
    \end{subfigure}
    \begin{subfigure}[c]{\textwidth}
        \centering
        \includegraphics[width=.75\textwidth]{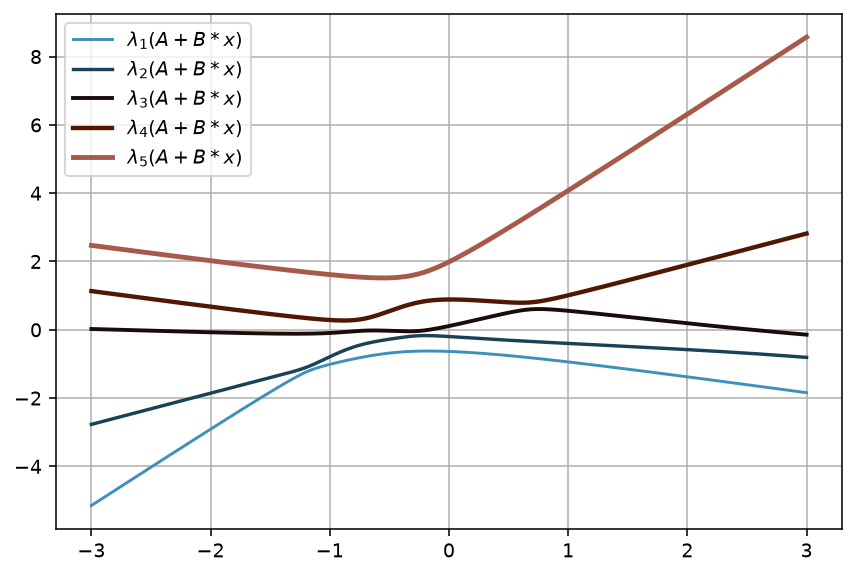}
        \caption{All models in the same axis}
    \end{subfigure}
    \caption{A set of models $\lambda_k(\mA + x \mB)$ for different values of $k$ for $5 \times 5$ symmetric matrices.}
    \label{fig:1d_eigenvalues}
\end{figure}
First, we see that all functions are continuous.  Moreover, $f_1$ is concave, $f_5$ is convex, and the most ``wiggly'' function is $f_3$ that corresponds to the mid-eigenvalue. Finally, we can see from $\lambda_2$ that the function may have a ``kink'', meaning it is not necessarily smooth. These phenomena, of course, are not a coincidence, and we devote this section formally defining and studying them.

\subsection{Continuity}\label{sec:continuity}
For any two matrices $\mA,\mB \in \sS^d$ we have
\[
|\lambda_k(\mA) - \lambda_k(\mB)| \leq \|\mA - \mB\|_2,
\]
where $\|\cdot\|_2$ is the spectral norm operator. This is well-known result in matrix perturbation theory \citep[Corollary~IV.4.9]{stewart1990matrix}. Consequently, we have the following corollary.
\begin{corollary}\label{thm:glob_sensitivity}
For any $\vx, \bm{\delta} \in \sR^n$, we have:
\begin{equation}\label{eq:gen_glob_sensitivity}
|f_k(\vx + \bm{\delta};\mA_{0:n}) - f_k(\vx;\mA_{0:n})| \leq \left \| \sum_{i=1}^n \delta_i \mA_i \right \|_2 \leq \sum_{i=1}^n |\delta_i| \|\mA_i\|.
\end{equation}
In particular, when $\bm{\delta} = \ve_i$, the $i$-th unit vector, we have:
\begin{equation}\label{eq:gen_feat_sensitivity}
|f_k(\vx + \ve_i;\mA_{0:n}) - f_k(\vx;\mA_{0:n})| \leq \| \mA_i \|_2.
\end{equation}
\end{corollary}
Corollary \ref{thm:glob_sensitivity} is a rare combination of a powerful result with a trivial proof. It shows us that $f_k$, as a function of $\vx$, is Lipschitz continuous everywhere. The Lipschitz constant is governed by the spectral norms of the matrices $\mA_0, \ldots, \mA_i$. And in particular, the spectral norm of $\mA_i$ directly exposes a global feature-influence bound.

Thus, in the context of machine learning we treat $\|\mA_i\|_2$ as a global feature-influence bound, analogous to the magnitude of the $i$-th coefficient in a simple linear model. In particular, when training, we may obtain a Lasso-like feature selection effect by adding $\sum_{i=1}^n \|\mA_i\|_2$ as a regularization term.

\subsection{Universality}
The last section shows our model family cannot represent functions that are not globally Lipschitz. In particular, we cannot represent discontinuous functions. But this leaves us with the important question - what functions \emph{can} we represent?

Let's start from the case of extremal eigenvalues. It is well-known that for $\mA \in \symd$ we have
\[
\lambda_{\min}(\mA) = \min_{\vx} \{ \vx^T \mA \vx : \|\vx\|_2 = 1 \},
\]
and 
\[
\lambda_{\max}(\mA) = \max_{\vx} \{ \vx^T \mA \vx : \|\vx\|_2 = 1 \}.
\]
The above equations are known as the \emph{Rayleigh--Ritz variational principle} \citep[Theorem~4.2.2]{horn2013matrix}. In particular, the function $q(\mA, \vx)=\vx^T \mA \vx$ is a \emph{linear} function of $\mA$, and thus $\lambda_1$ is concave as a minimum of linear functions, whereas $\lambda_d$ is convex as a maximum of linear functions. Consequently, we have the following result:
\begin{claim}\label{thm:shape}
    The model $f_1(\vx, \mA_{0:n})$ as defined in \eqref{eq:model_family} is a concave function of the feature vector $\vx$, whereas $f_d(\vx, \mA_{0:n})$ is a convex function of $\vx$.
\end{claim}
In fact we have an even stronger result - this family comprises universal approximators for convex and concave functions on compact domains, as shown in \citet{orielly2023spectrahedral}. Formally stated,
\begin{claim}\label{thm:max_eig_universal}
   Let $\sK$ be a compact convex subset of $\sR^n$, let $h: \sK \to \sR$ be convex, and let $\varepsilon > 0$. Then, there exists a sufficiently large $d$ and $\mA_0, ..., \mA_n \in \symd$ such that $|h - f_d(\cdot, \mA_{0:n})|_\infty \leq \epsilon$. 
\end{claim}
But what about the remaining eigenvalues? Cook et. al. in \cite{cook2025parametric} showed that for any internal eigenvalue $1 < k < d$, where $d$ is the matrix dimension, the family $f_k$ we consider in this paper comprise universal approximators for all continuous functions on compact domains. Thus, in terms of approximation power, the family is similar to multilayer perceptrons. 

\subsection{Orthogonal invariance}\label{sec:orth_invariance}
A well known linear algebra property is that for any orthogonal matrix $\mQ$ and any symmetric matrix $\mA$, we have
\[
    \lambda_k(\mQ^T \mA \mQ) = \lambda_k(\mQ).
\]
Thus, a spectral neuron parametrized by $\mA_0, \ldots, \mA_n$ represents exactly the same function as a neuron parametrized by $\mQ^T \mA_0 \mQ, \ldots, \mQ^T \mA_n \mQ$. In other words, the parametrization invariant under orthogonal similarity transformation.

Since any symmetric matrix $\mA$ has an eigenvalue decomposition $\mA = \mQ^T {\bm \Lambda} \mQ$ with $\bm \Lambda$ being diagonal, we may w.l.o.g choose one of the learned metrices to be diagonal without loosing expressive power. For example, we may choose $\mA_0$ to be diagonal, and the spectral neuron can be equivalently posed as
\[
\vx \to \lambda_k \left( \operatorname{diag}(\va_0) + \sum_{i=1}^n x_i \mA_i \right),
\]
with the vector $\va_0$ and the matrices $\mA_1, \ldots, \mA_n$ being its learnable parameters.

Invariance under orthogonality is not unique to the spectral neuron. For example, matrix factorization models \citep{srebro2003matrixFactorization,thurstone1931factoranalysis} are also invariant under (slightly different) orthogonality transformations.

\subsection{Latent-variable model}
Universality is an important theoretical property, but it does not lead to us actually understanding what the model does and being able to explain it to either ourselves or stakeholders. One way to gain insight into the model is using a characterization of eigenvalues as solutions of optimization problems.

First, consider the Courant variational characterization of the $k$-th smallest eigenvalue \citep{horn2013matrix}:
\begin{claim}
    Let $\mA \in \symd$. Then,
    \[
    \lambda_k(\mA) = \max_{\mC \in \sR^{(k-1)\times d}} \min_{\vu \in \sR^d} \left\{ \vu^T \mA \vu : \|\vu\|_2 = 1,\, \mC \vu = \vzero \right\}.
    \]
\end{claim}
Consequently, our model family can be written as
\[
f_k(\vx, \mA_{0:n}) = \max_{\mC \in \sR^{(k-1)\times d}} \min_{\vu \in \sR^d} \left\{
    \vu^T \left(\mA_0 + \sum_{i=1}^n x_i \vu^T \mA_i \vu\right) \vu : \| \vu\|_2 = 1, \, \mC \vu = \vzero
\right\}.
\]
In other words, $f_k$ is a \emph{latent variable model}. A model produces a latent matrix $\mathcal{A}(\vx)$, which is then used to solve an optimization problem with the latent variables $\mC$ and $\vu$ in the form of a $\min-\max$ game between two players. The first ``outer'' player chooses the latent matrix $\mC$, and in response, the second ``inner'' player chooses a latent unit vector $\vu$ orthogonal to all its rows, and pays $\vu^T \mathcal{A}(\vx) \vu$ to the first player. Naturally, the first player aims to maximize their revenue, whereas the inner player aims to minimize the incurred cost. 

Beyond mathematical theory, this can also be used to explain the model to a stakeholder, or a regulator. For example, consider modeling insurance risk for a client described by $\vx$. We may think of $\mathcal{A}(\vx)$ as a set of latent \emph{skills} to avoid filing a claim, where each matrix $\mA_i$ fulfills the latent skills associated with one feature that are summed up by $\mathcal{A}(\vx)$. Each entry $(i, j)$ corresponds to a skill dealing with a pair of difficulties. Then, we simulate a game between an adversarial (hence $\max$) environment $\mC$ that yields test cases $\vu$, and observe how the client can best (hence $\min$) cope with these tests. Indeed, by computing
\[
\vu^T \mathcal{A}(\vx) \vu = \langle \mathcal{A}(\vx), \vu \vu^T \rangle.
\]
we see that we take \emph{pairs of challenges} $\vu \vu^T$, and compute their weighted sum with weights $\mathcal{A}(\vx)$. 

This analogy of latent skills, of course, is not precise, due to the invariance under orthogonal similarities we just saw. Indeed, the ``latent skill'' matrices have multiple equivalent representations, and thus there are many possible ``laten skills'' that lead to the same prediction, but we believe it is still a useful analogy for communicating what the model does.

\subsection{A kind of a recurrent neural network}
A different variational characterization of the sequence of eigenvalues of a symmetric matrix leads to an interpretation of our model family as (a kind of) a recurrent neural network.
\begin{claim}
    Let $\mA \in \symd$ with eigenvalues $\lambda_1 \leq \ldots \leq \lambda_d$ and corresponding eigenvectors $\vu_1, \ldots, \vu_d$. Then,
    \[
        \lambda_k = \min_\vu \left\{
            \vu^T \mA \vu : \|\vu\|_2 = 1, \, \langle \vu, \vu_1 \rangle = 0, \, \ldots, \, \langle \vu, \vu_{k-1}\rangle = 0 
        \right\}.
    \]
\end{claim}
In other words, the $k$-th smallest eigenvalue of $\mA$ is the minimum of $\vu^T \mA \vu$ among the unit vectors orthogonal to eigenvectors corresponding to previous eigenvalues. Thus, we can think of it as a recurrent process: computing each eigenvalue yields an eigenvector, and all eigenvectors up to $k-1$
are used to compute the $k$-th eigenvalue. The process is depicted in \figref{fig:recurrent_eigenvalue}.
\begin{figure}[tbhp]
\includegraphics[width=\textwidth]{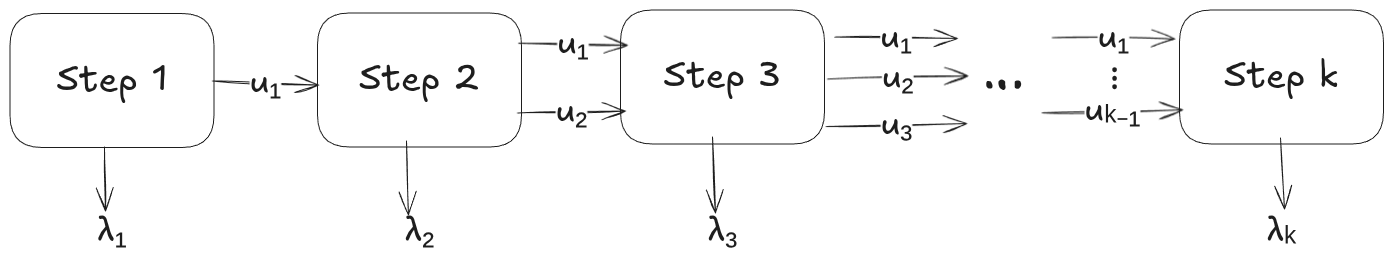}
\caption{Visualization of the $k$-th smallest eigenvalue as a recurrent neural network. Step $i$ computes the minimum of $\vu^T \mA \vu$ subject among unit vectors orthogonal to $\vu_1, \ldots, \vu_{i-1}$}
\label{fig:recurrent_eigenvalue}
\end{figure}

This gives us some intuition regarding the expressive power of the $k$-th eigenvalue. The ``simplest'' optimization problem, without any orthogonality constraints, gives us $\lambda_1$. Additional constraints introduce complexity to the optimization process, which schematically leads to a deeper recurrent network, and intuitively increase the expressive power of $f_k$. But this viewpoint may be deceiving, since there is a mirror-image of \emph{maximum} problems that yield the $j$-th \emph{largest} eigenvalue. Thus, from an intuitive perspective, the richest model corresponds to the mid eigenvalue for odd-sized matrices, and the two mid eigenvalues for even-sized matrices.

\subsection{A difference of convex functions}
Not only the largest eigenvalue, but also the sum of the largest eigenvalues,
\[
\Lambda_k(\mA) = \sum_{i=k}^d \lambda_i(\mA),
\]
is a convex function. This can be shown using Ky Fan's variational principle \cite{horn2013matrix}.
\begin{claim}[Ky Fan variational principle]
    Let $\mA \in \symd$. Then,
    \[
    \Lambda_k(\mA) = \max_{\mU \in \sR^{d \times (d-k+1)}} \left\{
        \operatorname{tr}(\mU^T \mA \mU) : \mU^T \mU = \mI
    \right\}.
    \]
\end{claim}
By the Ky Fan principle, the sum of the top eigenvalues $\Lambda_k = \lambda_k + \ldots + \lambda_d$ is also a maximum of \emph{linear} function of $\mA$, and thus convex. Thus, we can obtain an explicit formula for our model $f_k$ as a difference of convex functions:
\[
f_k(\vx, \mA_{0:n}) \equiv \lambda_k(\mathcal{A}(\vx)) = \Lambda_k(\mathcal{A}(\vx)) - \Lambda_{k+1}(\mathcal{A}(\vx)).
\]

Differences of convex (DC) functions admit dedicated optimization theory and algorithms \citep{lethi2018dcprog}, and software for solving optimization problems involving DC functions \citep{shen2016disciplined}. This formulation is mainly useful in learn-then-optimize scenarios, where our model comprises a cost of some sort that we later aim to minimize or maximize at inference.

\subsection{Monotone functions}\label{sec:monotonicity}
Recall that a matrix $\mP \in \symd$ is called \emph{positive semi-definite} if for any nonzer ovector $\vu$ we have $\vu^T \mP \vu \geq 0$. The definition of negative semi-definite matrices is analogous. 

For a positive semi-definie $\mP \in \symd$ and $\mA \in \symd$, we have 
\[
\lambda_k(\mA + \mP) \geq \lambda_k(\mA).
\]
This is known as eigenvalue monotonicity, and can be immediately seen from the formulations of $\lambda_k$ as optimization problems involving a quadratic cost, since
\[
\vu^T (\mA + \mP) \vu = \vu^T \mA \vu + \underbrace{\vu^T \mP \vu}_{\geq 0} \geq \vu^T \mA \vu.
\]
This powerful property gives us the ability to model monotone functions. For example, if $\mA_7$ is positive semi-definite and $\mA_{17}$ is negative semi-definite, the learned model model $f_k$ is, by construction, a nondecreasing function of $\evx_7$ and a nonincreasing fnuction of $\evx_{17}$, while being unrestricted in shape in the remaining features.

In \figref{fig:monotone_1d} we can see univariate functions defined by two matrices $\mA, \mB \in \sS^6$ whose entires were sampled uniformly at random in $[0, 1]$, and then $\mB$ was transformed to a positive-semidefinite matrix by clipping its eigenvalues to $[0, \infty)$. Indeed, all functions are non-decreasing, with $f_1$ being concave and non-decreasing, whereas $f_6$ is convex and non-decreasing.

\begin{figure}[tbhp]
    \centering
    \begin{subfigure}[c]{\textwidth}
        \centering
        \includegraphics[width=.75\textwidth]{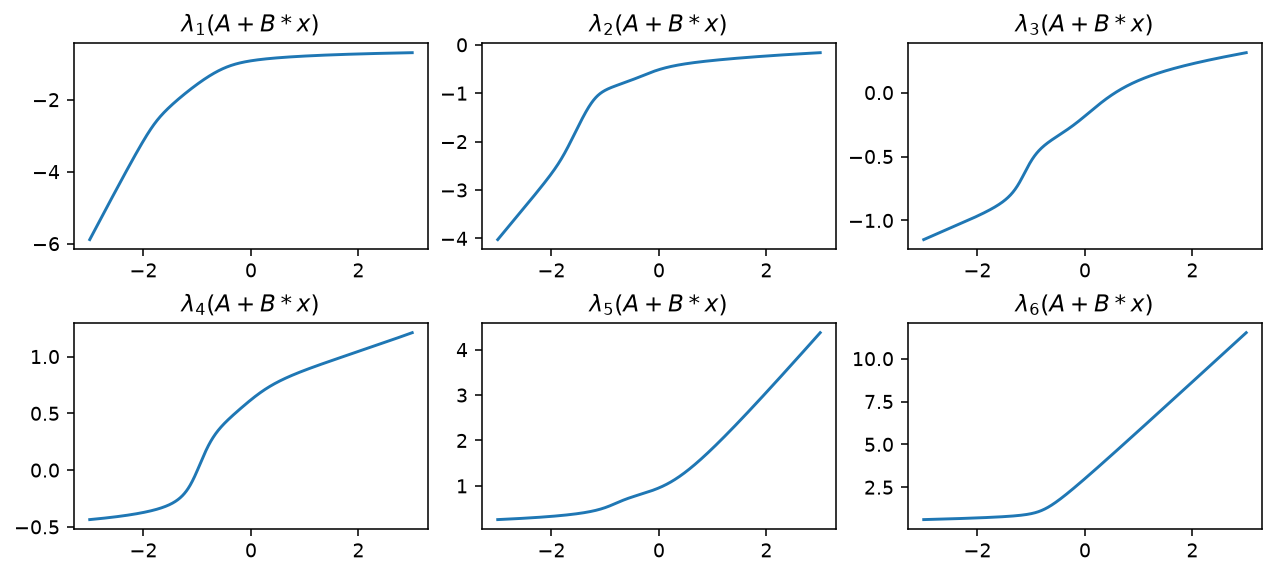}
        \caption{Gallery plot - each model in its own axis}
    \end{subfigure}
    \begin{subfigure}[c]{\textwidth}
        \centering
        \includegraphics[width=.75\textwidth]{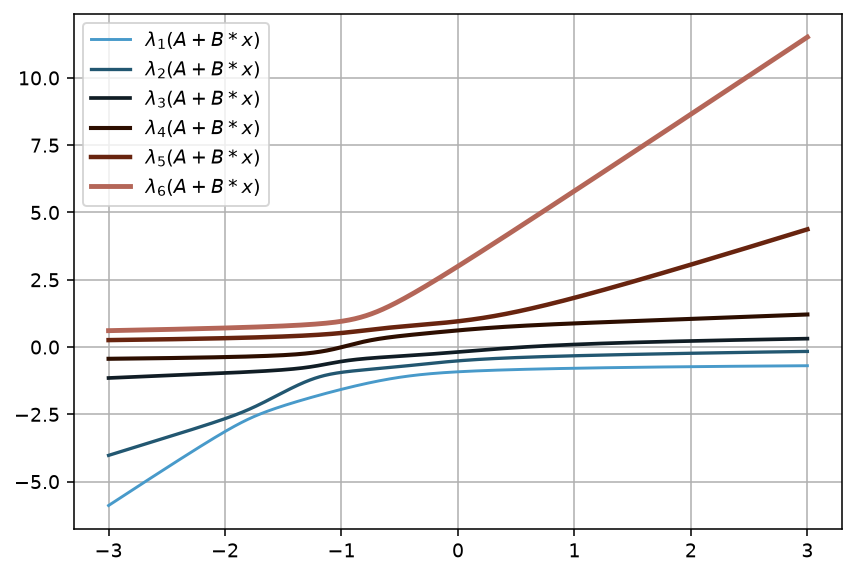}
        \caption{All models in the same axis}
    \end{subfigure}
    \caption{A set of models $\lambda_k(\mA + x \mB)$ for different values of $k$ for a general $\mA \in \sS^6$ and a positive semi-definite $\mB \in \sS^6$.}
    \label{fig:monotone_1d}
\end{figure}

\subsection{Summary}
In this section we've seen that modeling with $f_k$ provides control over the learned function and coefficient transparency through the learned matrices, much as linear models provide such transparency through their coefficients. Unlike linear models, the spectral neuron can model functions of arbitrary accuracy by scaling up the matrix dimensions. This unique combination of properties is rare to find. For example, deep neural networks, while possessing remarkable scaling properties, require extensive ``acrobatics'' to model convex functions, e.g., \emph{input-convex neural netwoks} by \citet{amir2017inputconvex}, or to partially monotone functions, e.g., \emph{deep lattice networks} by \citet{you2017deeplattice}, or even to make the influence of individual input features directly readable from their parameters.

However, the theoretical potential of the model family in scaling up with matrix size is not enough. Perhaps we aren't even able to learn the theoretically-possible models, or we might need matrices of prohibitively enormous size? To observe the behavior in practice we next study differential properties of the model family, to both gain insights, and the ability to train with standard machine-learning toolkits and their optimizers.

\section{Differentiation}
Learning a model $f_k$ from data requires being able to compute the gradient w.r.t the model's parametes, $\mA_{0:n}$. But apparently, from studying differentiation of $f_k$ we gain another insight - local feature influence. Namely, given a feature vector $\vx$ for which the model has made a prediction $f_k(\vx; \mA_{0:n})$, we can understand which features have the greatest and least local influence on the prediction for \emph{this specific instance} $\vx$.

\subsection{Generalized derivatives of $\lambda_k$}\label{sec:derivatives}
Since the $k$-th smallest eigenvalue is a non-smooth function, we cannot speak about its gradient, and we need some notion of generalized derivative. The notion of a Clarke subdifferential, which we define below, is a well accepted \citep{park2024what_ad_computes} notion of generalized derivatives used in automatic differentiation frameworks used in machine learning. One reason, for example, is the applicability of the chain rule.

\begin{definition}[Clarke directional derivative]
    Let $f$ be a function defined on an open set $\sD$.
    The Clarke directional derivative of a function $f$ at the point $\vx \in \sD$ is 
    \[
        f'(\vx; \vd) = \lim_{t \to 0+} \sup_{\vy \to \vx} \frac{f(\vy + t \vd) - f(\vy)}{t}
    \]
\end{definition}
This somewhat resembles the well-known directional derivative - the rate of change at $\vx$ in direction $\vd$, but we are not using the value of the function at $\vx$, and use a supremum in the neighborhood of $\vx$.

The Clarke sub-differential, which we denote by $\partial f(\vx)$, is defined in terms of the directional derivatives below.
\begin{definition}[Clarke sub-differential]
    Let $f$ be a function defined on an open set $\sD \subseteq \sR^m$. The Clarke sub-differential of a function $f$ at the point $\vx \in \sD$ is the set
    \[
    \partial f(\vx) = \{ \vs \in \sR^m: \sup_{\vd \in \sR^m} \{\langle \vs, \vd \rangle  - f'(\vx; \vd)\} \leq 0\}.
    \]
\end{definition}
Intuitively, it is the convex hull the linear slope models $\vs$ whose predicted rate of change $\langle \vs, \vd \rangle$  does not exceed the directional derivative in any direction $\vd$. 

A direct formula for the Clarke sub-differential of the $k$-th eigenvalue of a symmetric matrix has been  \citet{urruty1999clarkesubeigenvalue}, presented below.
\begin{claim}
    Let $E_k(\mA)$ be the \emph{eigenspace} associated with the $k$-th smallest eigenvalue of $\mA \in \symd$, namely, $E_k(\mA) = \{ \vv : \mA \vv = \lambda_k(\mA) \vv \}$. Then,
    \begin{equation}\label{eq:eigenvalue_subdifferential}
        \partial \lambda_k(\mA) = \operatorname{conv}\{ \vx \vx^T : \vx \in \E_k(\mA), \|\vx\|_2 = 1 \}.
    \end{equation}
\end{claim}
As a direct consequence of the above, for any normalized eigenvector $\vv$ associated with $\lambda_k(\mA)$ we have
\[
    \vv \vv^T \in \partial \lambda_k(\mA).
\]
The points of differentiability are exactly the ones where the eigenvalue is simple, there eigenspace dimension is one, and thus we have only \emph{one} element in the sub-differential, which is the gradient
\begin{equation}\label{eq:eigenvalue_gradient}
    \nabla \lambda_k(\mA) = \vv_k(\mA) \vv_k(\mA)^T.
\end{equation}
Differentiation facilitates learning, since we may now use any machine learning framework's optimizers, loss functions, and facilities. In PyTorch \citep{ansel2024pytorch}, for example, this is already implemented in the \texttt{torch.linalg.eigh} function. But occasionally we may want \emph{not} to use this function, since it computes all eigenvalues and eigenvectors and is unable to pick only one. This capability, for example, is provided by SciPy \citep{virtanen2020scipy}, and allows significant savings of computational resources when running on a CPU. In this case, we might want to implement our own differentiation routine.

But beyond learning, knowing this formulas provides another insight - local feature-influence analysis, as we see below.

\subsection{Local feature influence}\label{sec:local_sensitivity}
Suppose we obtained an instance $\vx$ and our eigenvalue model produced a prediction
\[
    f_k(\vx; \mA_{0:n}) = \lambda_k(\mathcal{A}(\vx)) = \lambda_k\left( \mA_0 + \sum_{i=0}^n x_i \mA_i \right).
\]
How do we know which feature has the greatest local influence on the prediction? For example - can we explain to a regulator why we believe the insured is of high risk? Which features had the greatest local influence on the predicted risk?

Suppose we landed at an eigenvalue of multiplicity 1, and thus $\lambda_k(\mathcal{A}(\vx))$ is differentiable and the corresponding normalized eigenvector is $\vv$. By the chain rule, we have
\[
\frac{\partial f_k}{\partial x_i} = \langle \vv \vv, \mA_i \rangle = \vv^T \mA_i \vv.
\]
In other words, the quadratic form $\vv^T \mA_i \vv$ is the signed local feature influence of feature $i$ on the prediction.

The computation becomes slightly more complex when the eigenvalue has a multiplicity greater than one - we need to be slightly more careful. It is hard to obtain an exact value for signed local feature influence, but the following Lemma makes computing a local feature-influence bound tractable.
\begin{lemma}\label{lem:local_sensitivity}
Let $\mV$ be a matrix whose columns are an orthonormal basis spanning the complete eigenspace of $\mathcal{A}(\vx)$ associated with the eigenvalue $f(\vx) = \lambda_k(\mathcal{A}(\vx))$. Then for every feature $i$,  we have
\[
    \sup_{\vg \in \partial f_k(\vx)} |g_i| \leq \| \mV^T \mA_i \mV\|_2,
\]
where $\partial f_k$ denotes the Clarke sub-differential of $f_k$. Consequently,
\[
    \limsup_{t \to 0} \frac{|f_k(\vx + t \ve_i) - f_k(\vx)|}{|t|} \leq \left\| \mV^T \mA_i \mV \right\|_2.
\]
\end{lemma}
In other words, we isolate an upper bounds for the magnitude of \emph{one} component of the sub-differential, and deduce a local feature-influence bound. The proof is slightly technical and not very insightful, and we show it below for completeness.
\begin{proof}
    Define the linear map $\mathcal{L}(\vh) = \sum_{i=1}^n \evh_i \mA_i$. Its adjoint is
    \[
        \mathcal{L}^*(\mG) = (\langle \mA_1, \mG \rangle, \ldots, \langle \mA_n, \mG \rangle). 
    \]
    The Clarke sub-differential of an ordered eigenvalue is \citep{urruty1999clarkesubeigenvalue}
    \[
        (\partial \lambda_k)(\mathcal{A}(\vx)) = \{ \mV \mZ \mV^T : \mZ \succeq 0, \operatorname{tr}(\mZ) = 1 \}. \tag{*}
    \]
    Since $\mathcal{A}(\vx)$ is affine and $\lambda_k$ Lipschitz, the Clarke chain rule \citep[Theorem~2.3.9]{clarke1990optimization} gives:
    \[
        \partial f_k(\vx) \subseteq \mathcal{L}^* (\partial f_k)(\mathcal{A}(\vx)).
        \tag{**}
    \]
    Combining (*) and (**) and using the rotation invariance of the trace operator, for every $\vg \in \partial f_k(\vx)$ there exists $\mZ \succeq 0$ with $\operatorname{tr}(\mZ) = 1$ such that 
    \[
        g_i = \langle \mA_i, \mV \mZ \mV^T \rangle =  \langle \mV^T \mA_i \mV, \mZ \rangle.
    \]
    By duality between the spectral and nuclear norms (generalized Cauchy-Schwartz),
    \begin{align*}
        |g_i| 
            &= |\langle \mV^T \mA_i \mV, \mZ \rangle| \\ 
            &\leq \| \mV^T \mA_i \mV \|_2 \| \mZ \|_* \\
            &= \| \mV^T \mA_i \mV \|_2.
    \end{align*}
    By \citet[Proposition 2.1.2]{clarke1990optimization}, the support function of $\partial f_k$ is its dirctional derivative, or formally,
    \[
        f_k'(\vx, \vh) = \max_{\vg \in \partial f_k(\vx)} \langle \vh, \vg \rangle.
    \]
    Therefore
    \[
        \max \{ f'_k(\vx, \ve_i), f'_k(\vx, -\ve_i) \} = \max \{ \max_{\vg \in \partial f_k(\vx)} \evg_i, \max_{\vg \in \partial f_k(\vx)} -\evg_i \} = \max_{\vg \in \partial f_k(\vx)} |g_i| \leq \| \mV^T \mA_i \mV \|.
    \]
    By definition of the Clarke directional derivative,
    \[
        \limsup_{t \to 0} \frac{|f_k(\vx + t \ve_i) - f_k(\vx)|}{|t|} \leq  \max \{ f'_k(\vx, \ve_i), f'_k(\vx, -\ve_i) \},
    \]
    proving the desired result.
\end{proof}

By Lemma \ref{lem:local_sensitivity}, a local feature-influence bound for the model $f_k = \lambda_k \circ \mathcal{A}$ under perturbation of feature $i$ can be computed as:
\begin{enumerate}
    \item Compute the spectral decomposition of $\mathcal{A}(\vx) = \mV^T \Lambda \mV$
    \item Compute range of eigenvalues equal to $\lambda_k(\mathcal{A}(\vx))$:
    \[m_- = \min_i \{\lambda_i(\mathcal{A}(\vx)) = \lambda_k(\mathcal{A}(\vx))\}, \quad m_+ = \max_i \{\lambda_i(\mathcal{A}(\vx)) = \lambda_k(\mathcal{A}(\vx))\},
    \] 
    and let $\mV_k$ be the matrix whose columns are $\mV_{m_-}, \mV_{m_- + 1}, \ldots, \mV_{m_+}$.
    \item Return the local feature-influence bound $\|\mV_k^T \mA_i \mV_k\|_2$.
\end{enumerate}

Finally, we note that the $\|\mV_k^T \mA_i \mV_k\|_2 \leq \| \mA_i \|_2$, and therefore the local feature-influence bound above is tighter than the global feature-influence bound. Consequently, whenever we \emph{can} use the local feature-influence bound to reason about a given prediction of the model, we probably should.

\section{Applications}
Below we discuss several potential applications of our model family.

\subsubsection*{Coefficient transparency with scaling}
The most direct application of the model family we study here is one that can improve with scaling of its size while retaining coefficient transparency analogous to that of linear models. Just as the magnitudes of linear-model coefficients expose feature influence, the spectral norms of the coefficient matrices $\mA_i$ expose feature-influence bounds. In \secref{sec:continuity} we saw that the spectral norm of $\mA_i$ is a global feature-influence bound, whereas in \secref{sec:local_sensitivity} the spectral norm of $\mV_k^T \mA_i \mV_k$, where $\mV_k$ is the matrix of the corresponding eigenvectors of $\mathcal{A}(\vx)$, is a local feature-influence bound associated with the specific instance $\vx$. However, unlike linear models, our family can grow in size by increasing matrix dimensions, and improve its representation power due to the universal approximation properties.

\subsubsection*{Shape-restricted models}
Suppose we partition the features $\{1, \ldots, n\}$ into three sets, $N_\uparrow \cup N_\downarrow \cup N_\star$, such that the model $f_k$ must be non-decreasing in $\evx_i$ for $i \in N_\uparrow$, non-increasing in $\evx_i$ for $i \in N_\downarrow$, and of arbitrary shape in the remaining features $N_\star$. 

As shown in \secref{sec:monotonicity}, we saw that by restricting $\mA_i$ for $i \in N_\uparrow$ to be positive semi-definite, and $\mA_i$ for $i \in N_\downarrow$ to be negative semi-definite, we obtain exactly the desired monotonicity properties. In practice, learning such a model can be done by parametrization: we can parametrize $\mA_i = \mL_i \mL_i^T$ for $i \in N_\uparrow$, and $\mA_i = -\mL_i \mL_i^T$ for $i \in N_\downarrow$, and learn the matrices $\mL_i$ by backpropagation.

Of course, we can mix and match with convexity and concavity. A convex model with $d$-dimensional matrices is obtained by $f_d$, whereas a concave model is obtained by $f_1$. By choosing an appropriate eigenvalue index and appropriate semi-definite matrix parameterizations we can obtain a convex / concave model that does not decrease in features $N_\uparrow$ and does not increase in features $N_\downarrow$.

\subsubsection*{Hyper-network setting}
Suppose that we have features $\vx \in \sR^n$ for which we need coefficient transparency / shape constaints, and features $\vy$ for which we do not need. We can build an arbitrarily-shaped neural network that predicts the symmetric matrices $\mA_0, \ldots, \mA_n$, which are the parameters of the spectral neuron $f_k(\evx; \mA_{0:n})$. Models which predicts the coefficients of other models are occasionally known as hyper-networks \citep{ha2016hypernetworks}.

As a concrete example, consider a model aiming to predict the probability of winning an auction given the auction opportunity features and bid. Such models are extensively used in online advertising by bid shading algorithms \citep{gligorijevic2020cikm,zhouKDD21}. Since the winning probability is a CDF, it is non-decreasing by definition. In this case $\vx$ consists of only one feature - the bid $b$, for which we care for the model's shape. The remaining features of the auction go into $\vy$, which is fed into an auction encoder that serves as a hyper-network. This network, in turn, predicts $\mA, \mB$ with a positive definite $\mB$ that are the weights of the model $\lambda_k(\mA + b \cdot \mB)$. This architecture is illustrated in \figref{fig:bid_hypernetwork}, and is a generalization of the architecture proposed in \citet{zhouKDD21}. Since $\vy$ can be arbitrary, it can even be a natural language description of the auction opportunity and the hyper-network can be a transformer-based encoder model \citep{devlin2019bert,liu2019robertarobustlyoptimizedbert,wang2024textembeddingsweaklysupervisedcontrastive}. Since the predicted $\mB$ is positive semi-definite, the model comprises a CDF by construction.

\begin{figure}[htbp]
    \centering
    \includegraphics[width=0.75\textwidth]{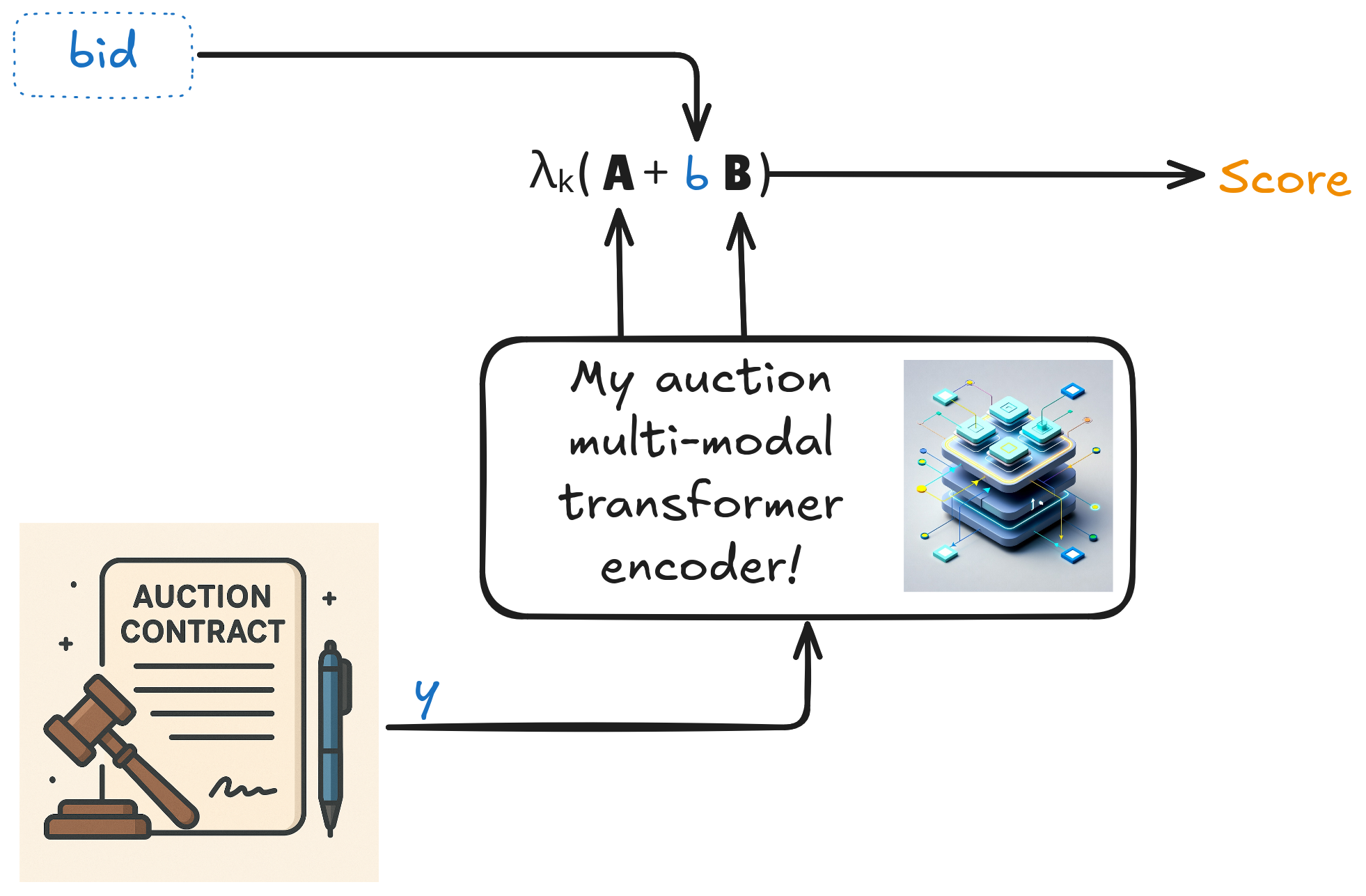}
    \caption{An auction winning probability (logit scale) prediction model using a hyper-network architecture. The auction contract $\vy$ is encoded into the symmetric matrices $\mA, \mB$ with $\mB$ being positive semi-definite. These, in turn, are parameters of an eigenvalue model $\lambda_k(\mA + b \cdot \mB)$ that predicts the winning probability on a logit scale.}
    \label{fig:bid_hypernetwork}
\end{figure}

\section{Training and initialization}
To train such models, beyond gradient computation that we typically do not have to manually do due to the abundance of autograd frameworks \citep{ansel2024pytorch}, we need a reasonable way to represent and initialize the matrices $\mA_0, \ldots, \mA_n$ of the spectral neuron. In this section we describe one such proposal. We do not aim for an exhaustive research into various techniques, but to propose \emph{one} which is guided by theory and worked well in our experiments, so we have good reasons to believe it to be a reasonable default starting point.

\subsection{Representation}
Dense symmetric matrices of size $d$ can be represented by a vector of their $D = \frac{1}{2} d (d + 1)$ upper (or lower) triangular elements. We propose a modification - parametrize a dense symmetric matrix by a vector $\vv \in \sR^D$ in the following manner:
\[
\sym(\vv) = \begin{pmatrix}
    \evv_1 & \frac{1}{\sqrt{2}} \evv_1 & \frac{1}{\sqrt{2}} \evv_2 & \frac{1}{\sqrt{2}} \evv_3 & \cdots & \frac{1}{\sqrt{2}} \evv_d \\
    \, & \evv_{d+1} & \frac{1}{\sqrt{2}} \evv_{d+2} & \cdots & & \frac{1}{\sqrt{2}} \evv_{2d-1} \\
    \, & \, & \evv_{2d} & \cdots & & \frac{1}{\sqrt{2}} \evv_{3d-3} \\
    \multicolumn{3}{c}{\text{symmetric}} & \ddots & \, & \vdots  \\
    \, & \, & \, & \, & \, & \evv_{d(d+1)/2}
\end{pmatrix}
\]
Namely, we embed the vector's entries into the matrix, taking the elements as is for the diagonal, but multiplying all off-diagonal elements by $\frac{1}{\sqrt{2}}$.

To understand why we chose this heuristic, consider the following $2 \times 2$ matrix:
\[
\begin{pmatrix}
    a & b \\ b & c
\end{pmatrix}.
\]
Its squared Frobenius norm is $a^2 + 2 b^2 + c^2$. However, if we were to parametrize it by a vector $(a, b, c)$, its squared Euclidean norm is $a^2 + b^2 + c^2$. Now, consider the matrix
\[
\sym(\vv) = \begin{pmatrix}
    \evv_1 & \frac{1}{\sqrt{2}} \evv_2 \\
    \frac{1} {\sqrt{2}} \evv_2 & \evv_3
\end{pmatrix},
\]
parametrized by $\vv = (\evv_1, \evv_2, \evv_3)$. Now, the Frobenius norm of the matrix $\sym(\vv)$, and the Euclidean norm of $\vv$ coincide. This is true in general - adding a $\frac{1}{\sqrt{2}}$ coefficient to the off-diagonal entries of the parametrized matrix makes the Euclidean norm of the vector $\vv$ and the Frobenius norm of $\sym(\vv)$ coincide. The reason we want to do it is avoiding an \emph{accidental} basis dependent preconditioner, namely, we want an optimizer making updates to $\vv$ to make similarly-scaled updates to the corresponding parametrized matrix $\sym(\vv)$. We do not claim this is a good or a bad preconditioner, we just want to avoid introducing one by mere coincidence.

In our experiments we also model non-decreasing functions. For simplicity, in our experiments we learn functions that are monotone increasing in only \emph{one} element feature vector $\vx$, which is by convention the last one - $\evx_n$. This may model, for example, the (logit of) the probability of winning an auction given the bid $\evx_n$, and additional features $\evx_1, \ldots, \evx_{n-1}$ that describe the auction itself. To enfore the shape constraint in the model, we require the learned matrix $\mA_n$ to be positive-semidefinite. 

There are many ways to parametrize positive-semidefinite matrices, but here we have an extreme simplification opportunity due to the orthogonal-invariance property described in \secref{sec:orth_invariance} - we may choose \emph{one} of the matrices $\mA_0, \ldots, \mA_n$ to be diagonal, and for such models we choose $\mA_n$ for that role.  A diagonal matrix is positive-semidefinite if and only if its diagonal entries are non-negative. To that end, we use the following smooth approximation of the \texttt{ReLU} function:
\[
\operatorname{squareplus}(x) = \frac{1}{2} \left(x + \sqrt{1 + x^2} \right).
\]
Given a vector $\vv \in \sR^d$, we parametrize a diagonal positive-semidefinite matrix as
\[
\operatorname{diag}(\operatorname{squareplus}(\vv)),
\]
where $\operatorname{squareplus}(\cdot)$ is applied component-wise. We could, in theory, choose any smooth \texttt{ReLU} approximation, such as the $x \to \ln(1+\exp(x))$. We chose $\operatorname{squareplus}(\cdot)$ since its gradients decay slower towards zero, which makes training possible even when $\operatorname{squareplus}(x)$ is very close to zero \citep{barron2021squareplus}. See \figref{fig:squareplus_deriv}. We leave devising a good parametrization for a larger number of semidefinite matrices that works well in practice the subject of future work.

\begin{figure}[htbp]
    \centering
    \includegraphics[width=0.5\linewidth]{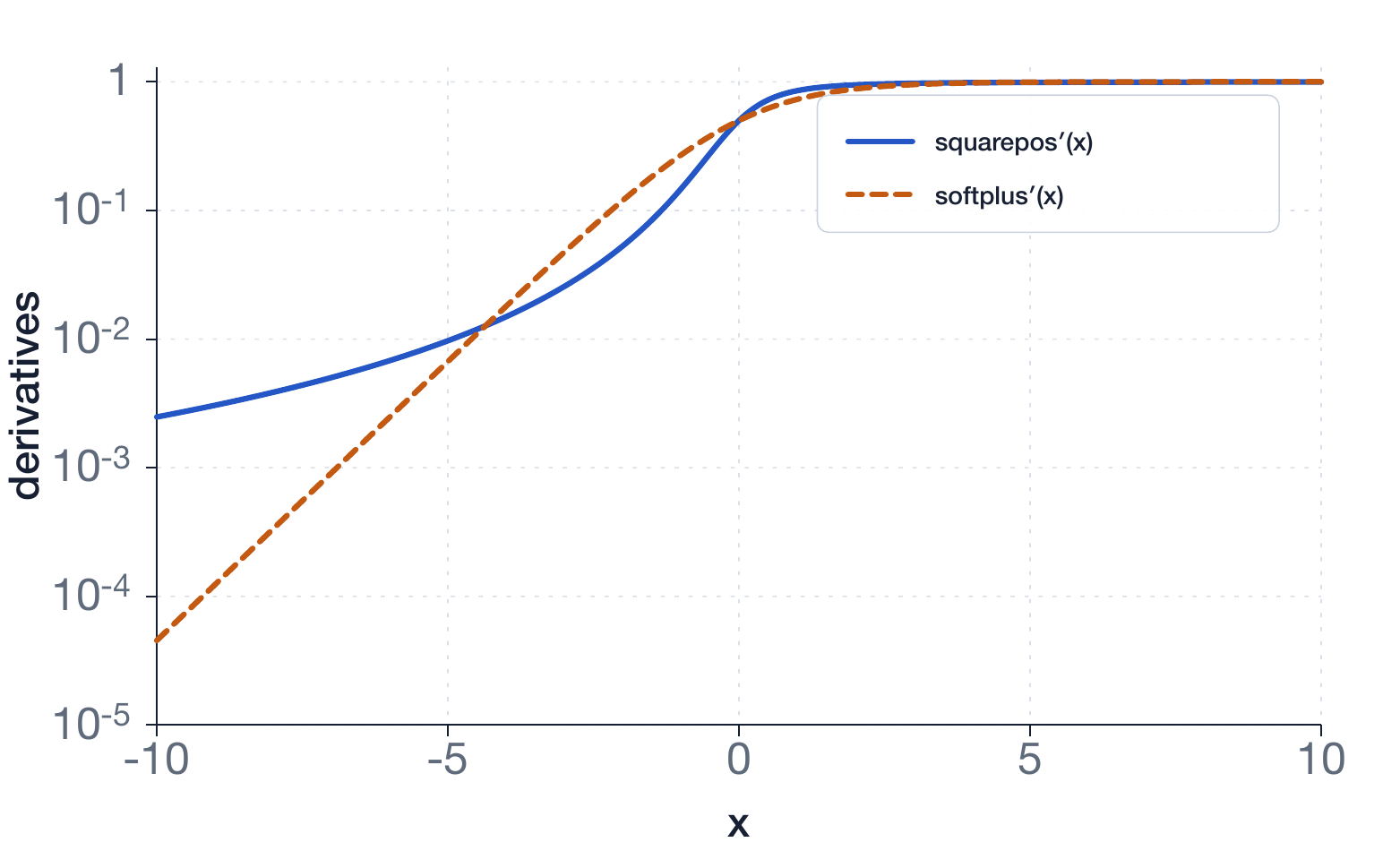}
    \caption{Derivatives of $\operatorname{squareplus}(x) = \frac{1}{2} \left(x + \sqrt{1 + x^2} \right)$ and $\operatorname{softplus}(x) = \ln(1+\exp(x))$}
    \label{fig:squareplus_deriv}
\end{figure}

To summarize, we parametrize an by the vectors $\vv_0, \ldots, \vv_n$ such that $\mA_i = \sym(\vv_i)$. When the model has one positive semi-definite matrix $\mA_n$, we parametrize $\mA_i = \sym(\vv_i)$ for $1 \leq i < n$, and $\mA_n$ is diagonal with $\mA_n = \operatorname{diag}(\operatorname{squareplus}(\vv_n))$. In both cases, the learned parameters are the vectors $\vv_0, \ldots, \vv_n$.

\subsection{Initialization}
Uniformly random or normal initialization, with an appropriate scale, could be a reasonable strategy, but here we propose a more carefully crafted idea that is inspired by our understanding of the training dynamics. Our objective is, of course, not studying initialization schemes, but rather propose \emph{one} that works reasonably well and can serve as a good default. 

Throughout this section we present two aspects of training dynamics that lead us to our desired initialization heuristic. However, it's important to note that despite being driven by theory, it is just a heuristic that we found to work well.

\subsubsection{The simultaneous-diagonalization trap}
Suppose that the matrices $\mA_0, \ldots, \mA_n$ are \emph{simultaneously diagonalizable}, meaning $\mA_i = \mQ \operatorname{diag}(\vlambda_i) \mQ^T$ for some orthogonal matrix $\mQ$ vectors $\vlambda_0, \ldots, \vlambda_n$. Consequently,
\[
    \lambda_k(\mathcal{A}(\vx)) = \lambda_k \left(\mQ \operatorname{diag}\left(\sum_{i=0}^n x_i \vlambda_i\right) \mQ^T\right) = \left(\sum_{i=0}^n x_i \vlambda_i\right)_{[k]}.
\]
In other words, the model predicts the piece-wise linear function corresponding to the $k$-th smallest entry of the vector $\sum_{i=0}^n x_i \vlambda_i$. By the chain rule applied to \eqref{eq:eigenvalue_gradient}, a sub-gradient w.r.t each $\mA_i$ is $\evx_i \vq \vq^T$, where $\vq$ is the corresponding column of the eigenvector matrix $\mQ$. 

Consequently, any optimization method that updated parameters as a linear combination of past gradients, including some methods used in modern machine learning, will always keep the matrices simultaneously diagonalizable with the same eigenvectors. Consequently, if at any time point of training a set of simultaneous diagonalizable matrices is obtained, they will remain that way from that point onward. Thus, to give our model a chance to learn functions that are not piece-wise linear, we should \emph{avoid simultaneous diagonalizability at initialization}.

\subsubsection{The eigenvalue gap effect}
A central theme in first-order optimization is that tighter bounds on deterministic gradient variation—expressed, for example, through Lipschitz or relative smoothness, and on stochastic -gradient variability, yield correspondingly stronger convergence guarantees \citep{bottou2018largescaleoptim}. Hence, our guiding principle is devising an initialization scheme that aims for low sensitivity of derivatives of $\lambda_k(\mA)$ to changes in $\mA$.

As shown in \secref{sec:derivatives}, matrices of the form $\vv \vv^T$, where $\vv$ is an eigenvalue of $\mA$ are Clarke subdifferentials of $\lambda_k(\mA)$, and in particular, this is the gradient whenever the eigenvalue is simple. A famous result on the sensitivity to change is the Davis Kahan theorem. To present it, we first need to define the notion of the \emph{eigengap}, which is the closest distance to nearby eigenvalues.
\begin{definition}[Eigengap]
    Let $\mA \in \symd$, and let $\alpha = \lambda_k(\mA)$ be an eigenvalue with multiplicity $m$, meaning, for some $r < s$ with $m = s - r + 1$
    \[
        \alpha = \lambda_r(\mA) = \lambda_{r + 1}(\mA) = \ldots = \lambda_s(\mA).
    \]
    Define $\lambda_0(\mA) = -\infty$ and $\lambda_{d+1}(\mA) = +\infty$, define the $k$-th eigengap as
    \[
    \operatorname{\gamma}_k(\mA) = \min(\alpha - \lambda_{r-1}(\mA), \lambda_{s+1} - \alpha).
    \]
\end{definition}
Below is one well-known variant of the Davis-Kahan theorem that gives us the necessary insight.
\begin{theorem}[Davis-Kahan \citep{yu2015daviskahan}]
Let $\mA, \mB \in \symd$, and assume $\lambda_r(\mA) = \ldots = \lambda_s(\mA)$ is an eigenvalue with multiplicity $m = s - r + 1$. Let $\mU \in \sR^{d \times m}$ be the matrix whose columns are orthonormal eigenvectors of $\alpha$, and let $\mV \in \sR^{d \times m}$ be the matrix whose columns are the orthonormal eigenvectors of $\lambda_r(\mB), \ldots, \lambda_s(\mB)$. Then for any $r \leq k \leq s$ we have
\[
\| \mU \mU^T - \mV \mV^T \|_2 \leq \frac{2 \| \mA - \mB \|_2}{\gamma_k(\mA)}.
\]
In particular, when $m = 1$, the matrices $\mU, \mV$ are column vectors of corresponding eigenvectors $\vu, \vv$, and 
\begin{equation}\label{eq:davis_kahan_simple}
\| \vu \vu^T - \vv \vv^T \|_2 \leq \frac{2 \| \mA - \mB \|_2}{\gamma_k(\mA)}.
\end{equation}
\end{theorem}
Although the Davis–Kahan theorem provides an upper bound, its inverse-eigengap dependence is sharp: examples attain the bound up to universal constants exist in \citet{yu2015daviskahan} and \citet{daviskahan1970rotation}. Thus, by \eqref{eq:davis_kahan_simple} we see that the eigengap $\gamma_k(\mA)$ plays a central role in the sensitivity of the gradient $\nabla \lambda_k(\mA)$ at points of differentiability to change - nearby matrices might have distant eigenvalue derivatives if $\gamma_k(\mA)$ is small. Thus, to give our model a chance to learn something meaningful quickly, we should aim to \emph{aviod small eigen-gaps at initialization}. 

\subsubsection{Initialization scheme}
Equipped with the two insights above, we would like an initialization scheme such that: (a) ensures eigenvalue gap at is controlled; and (b) the matrices are not simultaneously diagonalizable.

To that end, we initialize the matrix $\mA_0$ as 
\[
\mA_0
= \mQ^T \operatorname{diag}(
  -1, \ldots, -1,
  \underset{\mathclap{\substack{
    \uparrow\\[-0.9ex]\vert\\[-0.9ex]\vert\\[-0.2ex]\text{position }k
  }}}{0},
  1, \ldots, 1
)\mQ.
\]
where $\mQ$ is a random orthogonal matrix obtained from the QR decomposition of a matrix whose entries are chosen from $\mathcal{N}(0, 1)$, and the zero in the middle diagonal matrix is exactly $k$ - the index of eigenvalue of our spectral neuron. The matrices $\mA_1, \ldots, \mA_n$ are initialized as
\[
    \mA_i = \alpha_i \mI + \operatorname{diag}(\vepsilon_i),
\]
where $\alpha_i \sim \mathcal{U}(-\frac{1}{\sqrt{n}}, \frac{1}{\sqrt{n}})$, and each component $\evepsilon_{ij} \sim \mathcal{U}(-\frac{1}{20 n}, \frac{1}{20 n})$. When the feature vector is expected to be sparse and have at most $s$ entries (e.g., one-hot encoded columns of a tabular dataset), we use $s$ instead of $n$ in the above formulas.  The choice of the interval for $\alpha_i$ aligns with the widely used Kiaming uniform initialization scheme \citep{he2015dkiaminguniform}. The remaining choices are explained below.

First, note that without the jitter  $\operatorname{diag}(\vepsilon_i)$, letting $\va(\vx) = \sum_{i=1}^n \alpha_i \evx_i$, the eigenvalues of $\mathcal{A}(\vx)$ at initialization would be:
\begin{align*}
    \lambda_1(\mathcal{A}(\vx)) &= -1 + \va(\vx) \\
    \mathrlap{\vdots} \\
    \lambda_{k-1}(\mathcal{A}(\vx)) &= -1 + \va(\vx) \\
    \lambda_k(\mathcal{A}(\vx)) &= \va(\vx) \\
    \lambda_{k + 1}(\mathcal{A}(\vx)) &= 1 + \va(\vx) \\
    \mathrlap{\vdots} \\
    \lambda_{d}(\mathcal{A}(\vx)) &= 1 + \va(\vx)
\end{align*}
Thus, at initialization the eigenvalue $\lambda_k$ is well-separated from its neighbors. However, without the jitter all initial feature matrices are multiples of the identity matrix, and diagonalized by any orthogonal matrix $\mQ$. This contradicts our desire for the matrices \emph{not} to be simultaneously diagonalizable. The addition of the jitter term ensures that almost surely $\mA_0 \mA_i \neq \mA_i \mA_0$. Since matrices are simultaneously diagonalizable if and only if they commute pairwise, this indeed achieves our desired goal. 

The reason for the apparently peculiar interval from which $\vepsilon_i$ is chosen due to keep the eigenvalue separation property, as shown in the following Lemma.
\begin{lemma}\label{lem:eigengap}
    Let 
    \[
        \mA_0 = \mQ \operatorname{diag}(-1, \ldots, -1, 0, 1, \ldots, 1) \mQ^T,
    \]
    and for $1 \leq i \leq n$ let
    \[
        \mA_i = \alpha_i \mI + \operatorname{diag}(\varepsilon_{i1}, \ldots, \varepsilon_{in}),
    \]
    with $\alpha_i$ arbitrary, and $|\evepsilon_{ij}| \leq \frac{1}{4Rn}$. 
    
    Suppose that $\|\vx\|_\infty \leq R$. Then,
    \[
        \gamma_k(\mathcal{A}(\vx)) \geq \frac{1}{2}.
    \]
\end{lemma}
Before proving the Lemma, let's first discuss its meaning. Given a bound $R$ on the magnitude of our features, we should initialize the entries of the jitter vector $\vepsilon$ uniformly at random from $[-\frac{1}{4Rn}, \frac{1}{4Rn}]$ to ensure an eigen-value gap of at least $\frac{1}{2}$. Our peculiar choice corresponds to $R = 5$, which is a \emph{heuristic choice} corresponding to an assumption that features are typically standardized and approximately normal. Indeed, for a standard normal variable $\rvz$ we have $\mathcal{P}(|\rvz| > 5) \approx 5.73 \times 10^{-7}$, and thus by union bound if $\vx \in \sR^n$ has standard normal entries then
\[
    \mathcal{P}(\|\vx\|_\infty \leq 5) \geq 1 - 5.73 \times 10^{-7} n,
\]
which is close to $1$ for numbers of features we can expect in tabular datasets in practice. Of course, this is just a heuristic, and devising an optimal initialization scheme, or testing several candidates, is out of the scope of this paper. Below is the elementary proof of the Lemma.
\begin{proof}
    By the Lemma's assumptions, we have 
    \[
        \mathcal{A}(\vx) = \underbrace{\mQ \operatorname{diag}(-1, \ldots, -1, 0, 1, \ldots, 1) \mQ^T + (\sum_{i=1}^n \alpha_i \evx_i) \mI}_{\mD(\vx)} + \underbrace{\sum_{i=1}^n \evx_i \operatorname{diag}(\vepsilon_i)}_{\mE(\vx)}.
    \]
    Since adding a multiple of the identity does not change eigenvalues, we have $\gamma_k(\mD(\vx)) = 1$. The only source of change in the eigenvalues, and thus eigenvalue gaps, comes from $\mE(\vx)$. Since $\mE(\vx)$ is diagonal, 
    \[
        \|\mE(\vx)\|_2 = \max_j \left| \sum_{i=1}^n \evx_i \evepsilon_{ij} \right| \leq \max_j \sum_{i=1}^n |\evx_i| |\evepsilon_{ij}| \leq R n \max_{ij} |\evepsilon_{ij}| \leq R n \cdot \frac{1}{4Rn} = \frac{1}{4}.
    \]
    By Weyl's inequality,
    \[
        \gamma_k(\mD(\vx) + \mE(\vx)) \geq 1 - 2 \|\mE(\vx)\|_2 \geq 1 - 2 \times \frac{1}{4} = \frac{1}{2}.
    \]
\end{proof}

A slightly different initialization scheme is needed for positive semi-definite matrices, which we chose to parametrize as non-negative diagonal matrices. We choose a similar scheme to unconstrained matrices, but we choose $\alpha_n$ uniformly at random in $[\frac{1}{2\sqrt{\mathtt{fan\_in}}}, \frac{1}{\sqrt{\mathtt{fan\_in}}}]$. The parametrization vectors $\vv_0, \ldots, \vv_n$ are initialized accordingly to produce the desired random matrices under the $\operatorname{squareplus}$ transformation. Since $|\varepsilon_{ij}| \leq \frac{1}{20n} < \frac{1}{2\sqrt{n}}$ for $n > 1$, the initialization scheme indeed produces a positive semi-definite (and in fact, positive definite) matrix.

\subsection{Computational complexity}
Eigenvalue problems are generally solved in $O(n^3)$ time, similarly to matrix multiplication. However, the devil is in the details - eigenvalue problems are significantly more expensive to solve in practice. Dense matrix multiplication requires roughly $2 d^3$ floating point operations (FLOPS), however solving a symmetric eigenvalue problem requires significantly higher effort. Moreover, the computational effort at training is different than the one at inference: when training, we need eigenvectors compute gradients; for inference, we just need the $k$-th eigenvalue.

Typical symmetric eigenvalue solvers, such as the LAPACK SYEVD \citep{anderson1999lapack} solver employed by PyTorch \citep{ansel2024pytorch} operate in phases: the matrix is first reduced to tri-diagonal form using roughly $\frac{4}{3} d^3$ FLOPS, then a divide-and-conquer algorithm is employed to find eigenvalues \citep{lang2000direct} costing additional $O(d^2)$ FLOPS \citep{tisseur1999parallel}. To compute eigenvectors, we need additional floating point operations to solve linear systems, totalling roughly between $\frac{14}{3} d^3$ FLOPS.

Thus, comparing to a simple linear model with $n$ features, whose computational cost is $2n$ FLOPS, here way pay $nd^2 + \frac{14}{3} d^3$ FLOPS at training, and $nd^2 + \frac{4}{3} d^3$ FLOPS at inference (assuming the above LAPACK routines). We trade off speed for a combination of expressivity and coefficient transparency. Moreover, training is significantly more expensive than inference.

We could, of course, use dedicated algorithms that use the fact that we require just \emph{one} eigenvalue and \emph{one} corresponding eigenvector for training, rather than using an existing machine-learning framework. However, the basic ideas stand: we still require an expensive symmetric eigenvalue solver, both for training and inference. Writing such custom training and inference kernels is out of the scope of this work.

\section{Numerical experiments}
Our main thesis is that we have a family of models that is useful because it combines a unique set of properties: shape control by construction, coefficient transparency, and improvement by scaling. We do not claim state of the art performance, and provide baselines to qualitatively understand the performance of spectral neurons compared to other well-known models.

Since shape control is achieved by construction, numerical tests for shape control provide no additional value. The ability to improve with scaling is possible theoretically because of the universal approximation properties, but needs to be verified in practice - the fact that the model can represent our desired function does not mean we can learn this representation from data. Finally, spectral norms provide global feature-influence \emph{bounds}, but these upper bounds might be arbitrarily far from reality in practice. Thus, we also need to evaluate the gap between the upper bounds and actual prediction deviations when inputs are perturbed.

\subsection{Performance experiments}

We evaluate the performance of the model family both on synthetic data and on real datasets. The performance experiments are designed as \emph{scaling experiments} - we plot the best performance we could achieve for a given model family after performing a given number of optimizer steps. The objective is showing that the spectral neuron is able to learn from data in a practically useful manner, by observing improvement with model size and improvement with amount of data seen.

Synthetic data and real dataset experiments differ slightly, but have a common protocol. A stream of $N$ randomly selected mini-batches of predefined size provided to several model training procedures $T_{s, p}$ corresponding to model initialization seed $s \in \mathrm{seeds}$ and optimizer hyper-parameters $p$. For simplicity, and to save compute power, we use Adam \citep{kingma2017adam} as our optimizer in all experiments, and its learning rate as the only hyper-parameter.

At each checkpoint $1 \leq c_1 < c_2 \ldots < c_k = N$, defined by a certain amount of optimizer updates, we select the hyper-parameters corresponding to the best \emph{median} performance on a validation set, among all randomly selected seeds. Then, we use those hyper-parameters to train several models, each initialized from a different seed, and report plot 25\%, 50\%, and 75\% quantiles of their test set performance. This way, we simulate hyper-parameter tuning for each check-point, and measuring test performance of the hyper-parameters that perform best at that checkpoint.

Synthetic experiments contribute additional randomness explained below. In any case, the median for tuning, and the tes quantiles for plotting are taken over \emph{all} randomness.

All experiments use the mid-eigenvalue index $k = \frac{d}{2}$, and we use only odd values of $k$ to make the middle index unambiguous.

\subsection{Univariate functions}
In this experiment we fit two kinds synthetic random \emph{univariate} functions of varying complexity: arbitrarily shaped, and monotone functions. 

The arbitrarily shaped functions are formed as interpolating cubic B-Splines with not-a-knot end conditions, passing through points whose $x$ coordinates are the regular grid in $[-4, 4]$, and $y$ coordinates are chosen from the standard normal distribution. The monotone functions are chosen similarly, as an piecewise cubic Hermite interpolating polynomial (PCHIP) function, that interpolates points whose $x$ coordinates are the same regular grid, while the $y$ coordinates are chosen to be a non-decreasing vector in the following manner: samples from a log-normal distribution are taken, their cumulative sums are computed, and the resulting sequence of cumulative sums is standardized to have zero mean and a stanard deviation of 1. Both kind of functions are illustrated in Figure \figref{fig:univariate_targets}. At each complexity level, we use several random functions chosen from several seeds to neutralize bias towards a specific function.
\begin{figure}[htbp]
    \centering
    \begin{subfigure}[c]{.4\textwidth}
        \centering
        \includegraphics[width=\textwidth]{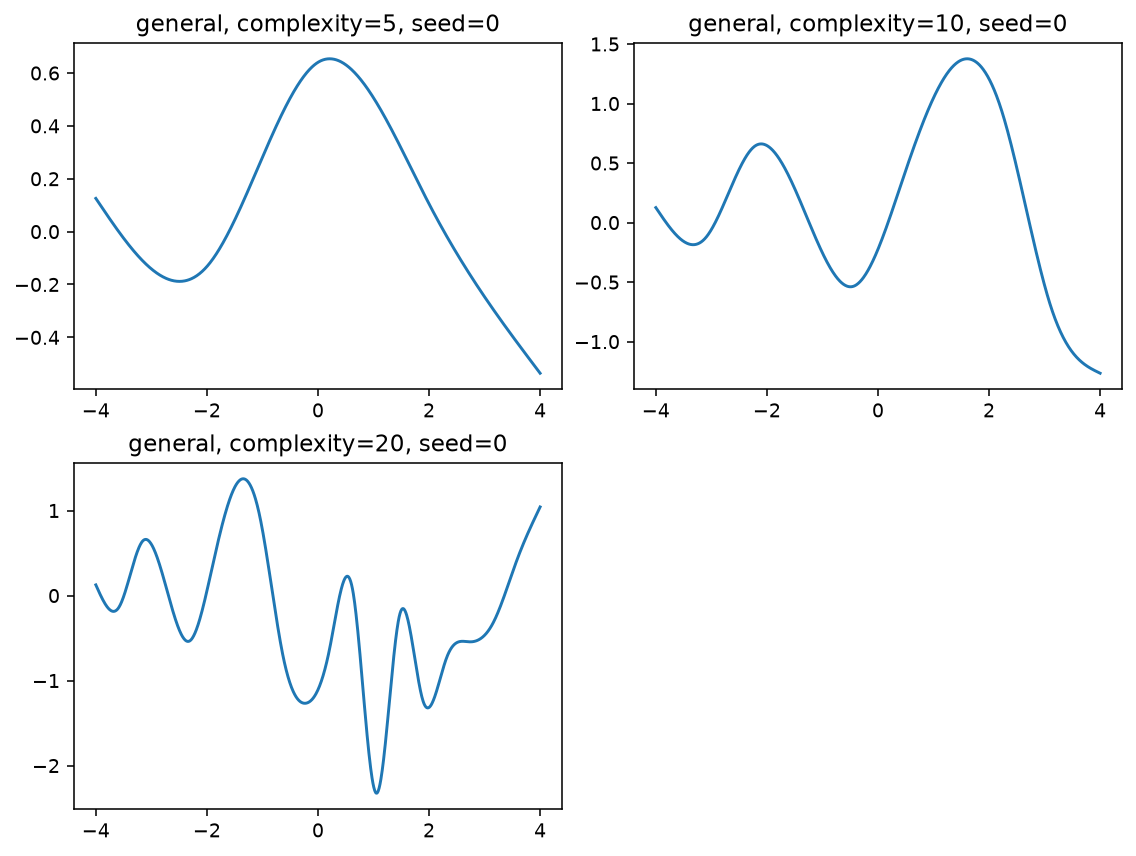}
        \caption{Arbitrary (general) target functions.}
    \end{subfigure}\qquad
    \begin{subfigure}[c]{.4\textwidth}
        \centering
        \includegraphics[width=\textwidth]{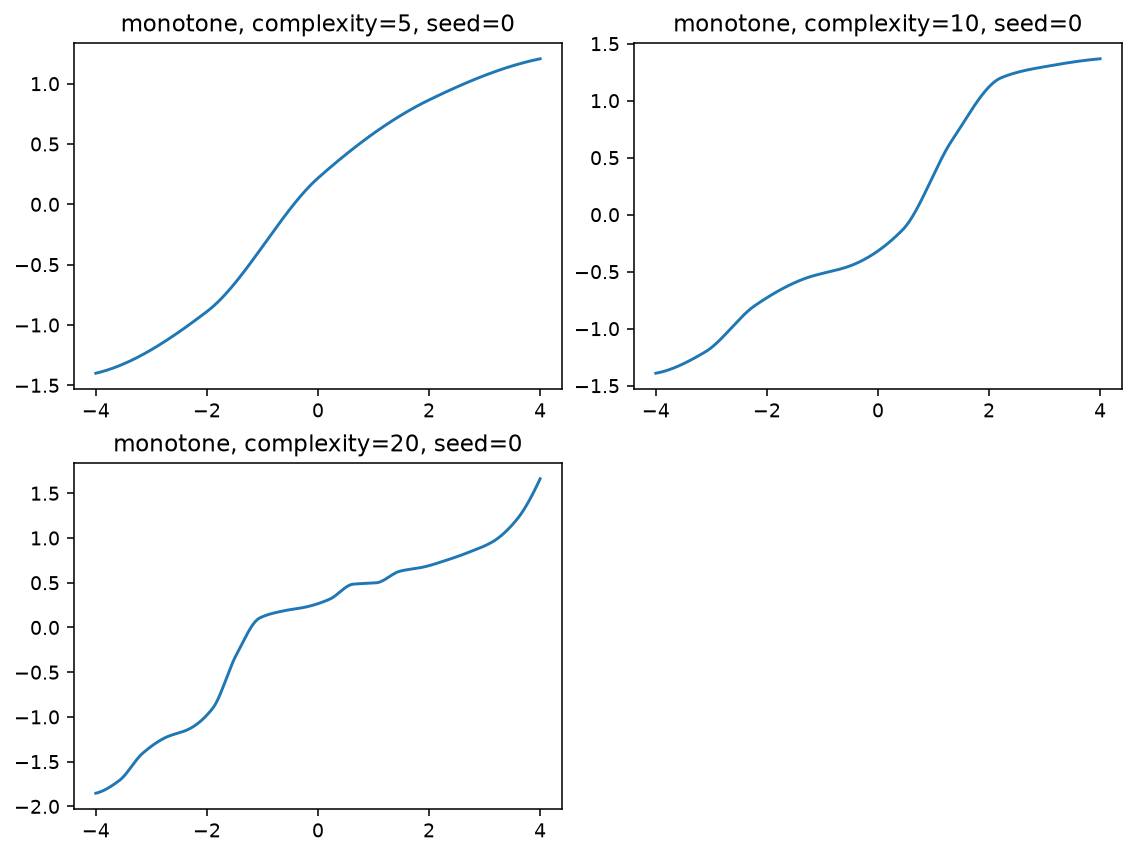}
        \caption{Monotone target functions.}
    \end{subfigure}
    \caption{Univariate fitting targets of varying complexity.}
    \label{fig:univariate_targets}
\end{figure}

Once the functions are determined, we have two variants of training procedures: noisless and noisy. In the noisless setting, the training data comprises of points of the form $(x, f(x))$ coordinates, where $f$ is the appropriate target function, and $x$ are chosen uniformly at random in $[-4, 4]$. In the noisy setting, the training data has points of the form $(x, f(x) + \varepsilon)$, where $\varepsilon \sim \mathcal{N}(0, 0.1)$. The validation set is randomly generated exactly like the training set, but the test set is always of the form $(x, f(x))$, where $x$ come from a regularly-spaced dense grid in $[-4, 4]$ of $N=10,000$ points. This choice stems from the fact that the validation set is a part of the training procedure, but the test set's job is to assess the ability to actually learn the target function. General functions are used to train unconstrained models, whereas monotone functions are used to train both unconstrained and monotone-by-design models.

Results of the scaling experiments are plotted in Figure \figref{fig:univariate_scaling_results}. We can observe several phenomena: (a) higher complexity functions require higher dimensional matrices to obtain a small error. Indeed, we can see this directly for general functions - the lower-dimensional model errors ``flatten out'' at some budget, while the higher-dimensional models keep improving. Consequently, the model's fitting power indeed improves with dimension; (b) monotonicity provides a useful inductive bias when the target is monotone. Indeed, at smaller budgets we already obtiain better performing models, especially with higher-complexity functions. Moreover, even low-dimensional monotone models continue improving with more training samples, and do not ``flatten out''.

\begin{figure}
    \centering
    \begin{subfigure}[c]{\textwidth}
        \includegraphics[width=\textwidth]{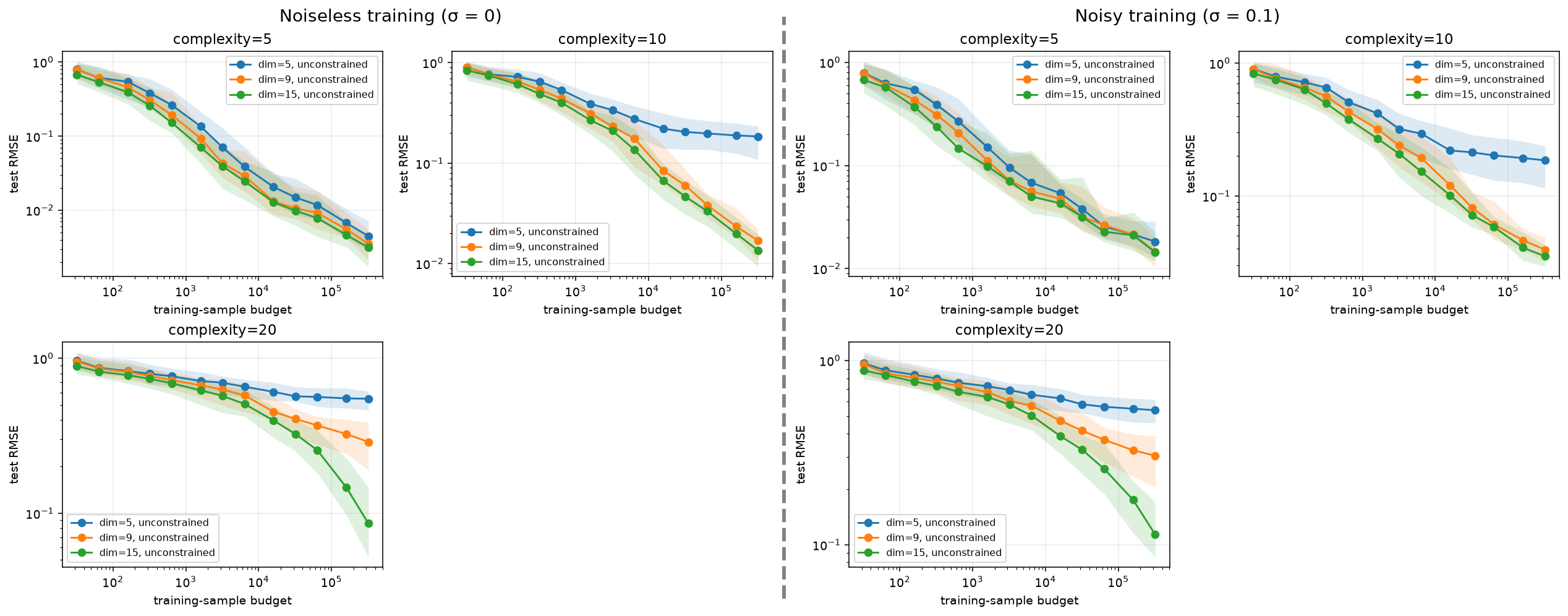}
    \end{subfigure}
    \rule{\textwidth}{0.4pt}
    \begin{subfigure}[c]{\textwidth}
        \includegraphics[width=\textwidth]{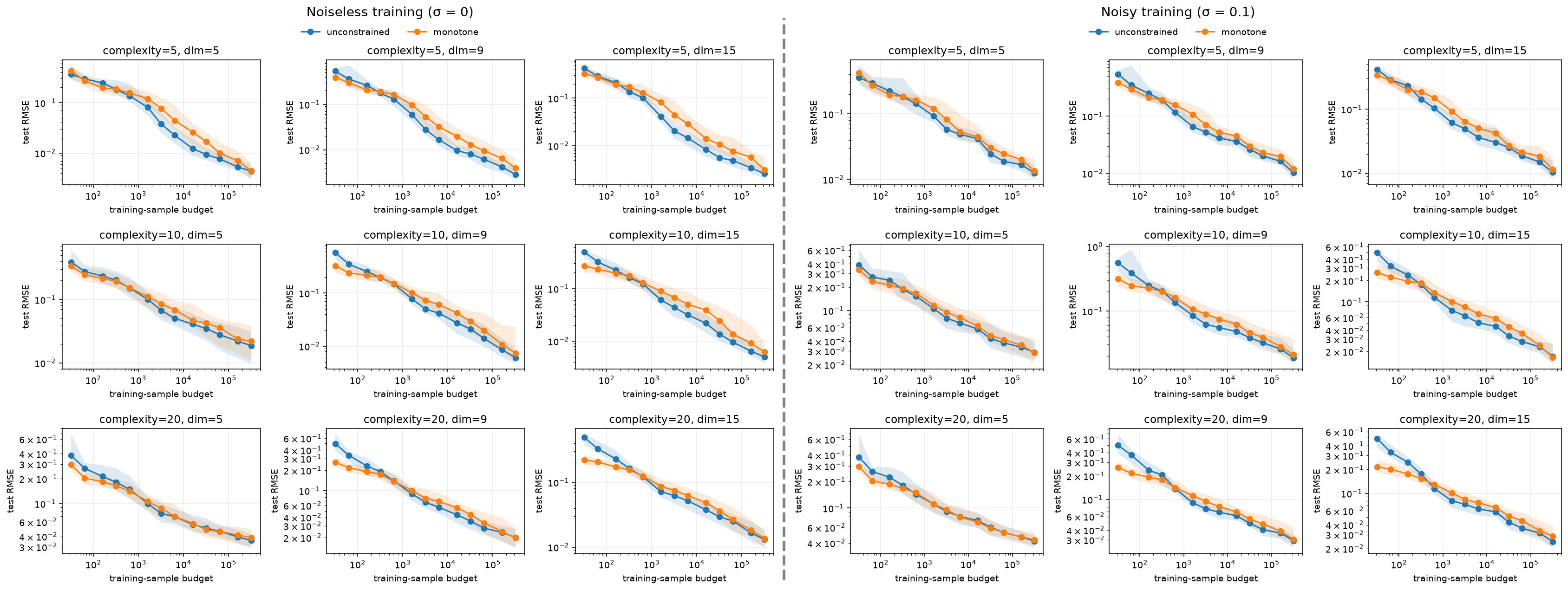}
    \end{subfigure}
    \caption{Univariate scaling results. $x$-axis - number of training samples. $y$-axis - test loss of the best configuration determined by the validation set. Top - generic targets, grouped by function complexity. Bottom - monotone targets, with complexity along the rows, and matrix dimensions along the columns. Left - noiseless labels. Right - noisy labels.}
    \label{fig:univariate_scaling_results}
\end{figure}

\subsection{Bivariate function}
This experiment's setup is identical to the univariate functions experiment, but the target functions are different. General functions $y = f(x_1, x_2)$ are obtained as an interpolating tensor-product cubic B-Spline on a uniform grid on $[-4, 4] \times [-4, 4]$, with $c \times c$ points along each dimension, where $c \in \sN$ the complexity. The $y$ coordinates at the interpolation points are chosen uniformly at random.

Monotone target functions are unconstrained in $x_1$ but monotone increasing in $x_2$. The construction is slightly more involved, since we are not aware of a simple interpolant preserving monotonicity only along one dimension. At the high level, we construct $m$ monotone $C^2$ functions $f_1(x_2), \ldots, f_m(x_2)$ using exactly the same process we obtained monotone univariate functions, to obtain monotone slices of uniformly spaced $x_1$ coordinates. Then, we blend the slices together using a $C^2$ function to obtain our target bivariate function. Formally, letting $-4 = \xi_1 < \xi_2 < \ldots < \xi_m = 4$ be uniformly-spaced grid points, to obtain $f(x_1, x_2)$ we determine the two grid points $\xi_i$ and $\xi_{i+1}$ such that either $x_1 \in [\xi_i, \xi_{i+1})$ for $i < m - 1$ or $x_1 \in [\xi_{m-1}, \xi_m]$ for $i = m - 1$, compute the relative distance from $\xi_i$ as $t = (x_1 - \xi_i) / (\xi_{i+1} - \xi_i)$, and blend $f(x_1, x_2) = (1 - w(t)) f_i(x_2) + w(t) f_{i+1}(x_2)$ with $w(t) = 6t^5 - 15 t^4 + 10 t^3$ being the unique lowest-degree polynomial satisfying the end-point conditions $w(0) = 0$, $w(1) = 1$, and $w'(0) = w'(1) = w''(0) = w''(1) = 0$. These endpoint conditions ensure the target function is $C^2$. The process is illustrated in Figure \ref{fig:slice_blending}, and examples of the resulting bivariate functions are shown in Figure \ref{fig:bivariate_targets}.
\begin{figure}[htbp]
    \centering
    \includegraphics[width=0.95\linewidth]{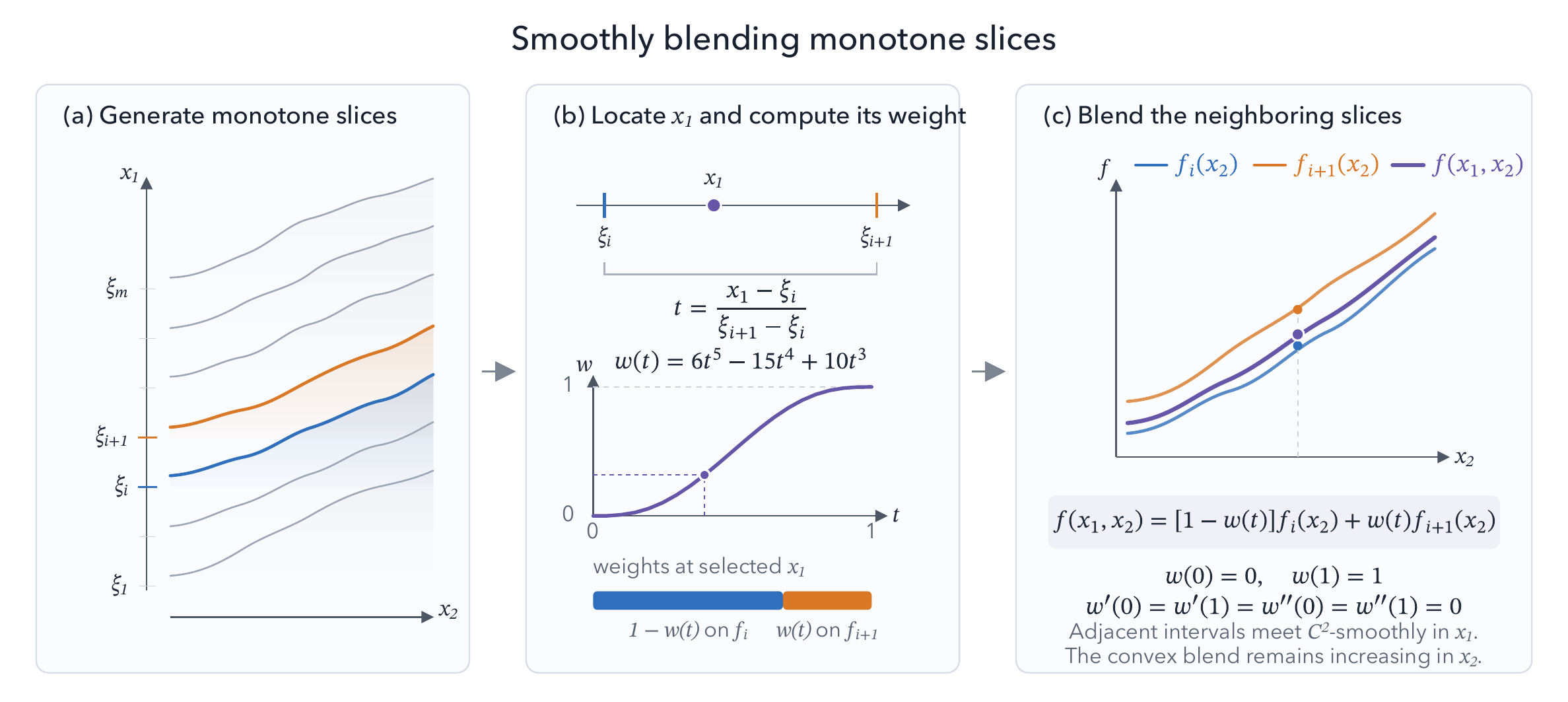}
    \caption{The process of obtaining a function $f(x_1, x_2)$ that is monotone in $x_2$ and unconstrained in $x_1$, by blending slices that are monotone in $x_2$.}
    \label{fig:slice_blending}
\end{figure}
\begin{figure}[htbp]
    \centering
    \begin{subfigure}[c]{.45\textwidth}
        \centering
        \includegraphics[width=\textwidth]{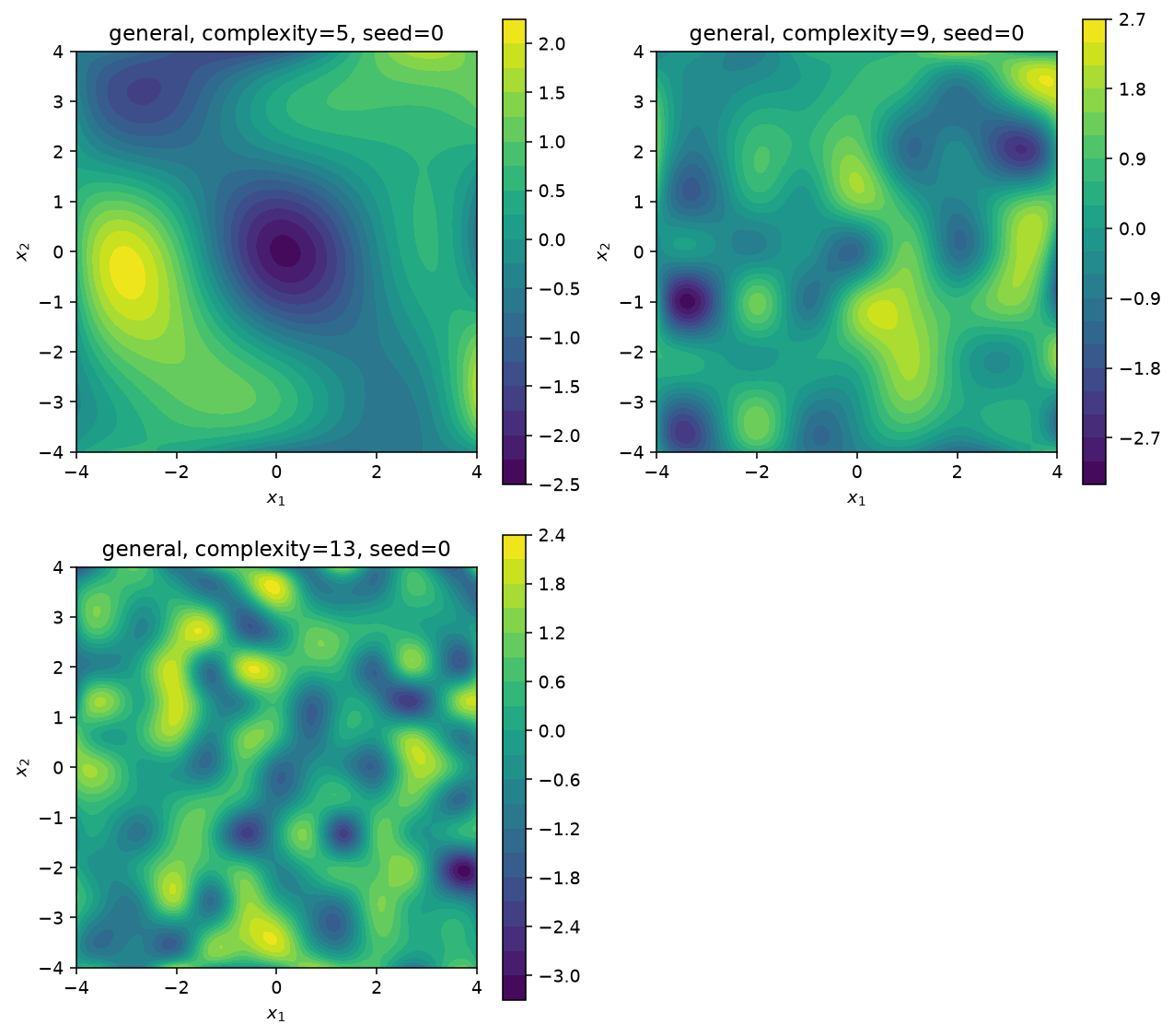}
        \caption{Arbitrary (general) target functions.}
    \end{subfigure}\qquad
    \begin{subfigure}[c]{.45\textwidth}
        \centering
        \includegraphics[width=\textwidth]{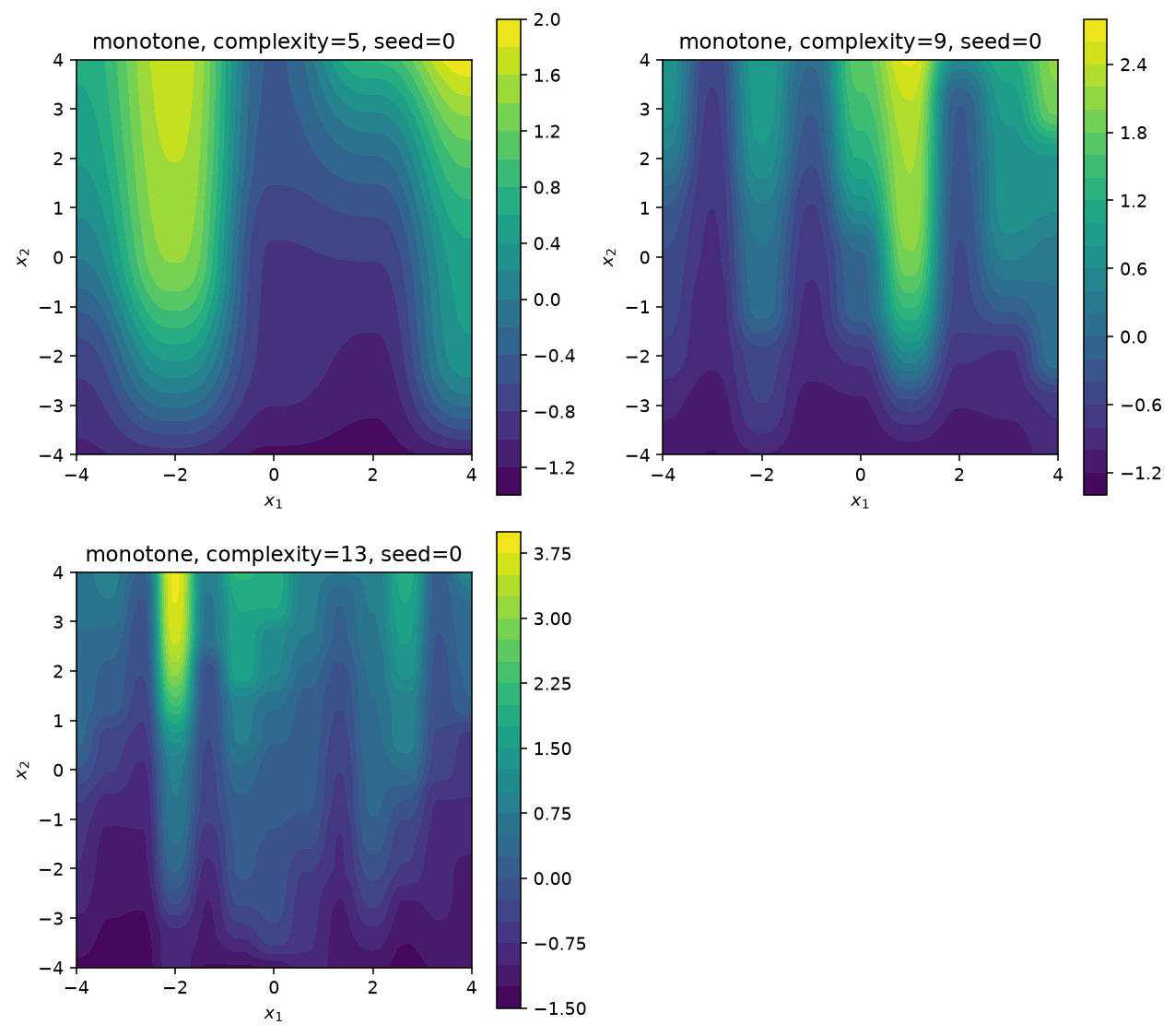}
        \caption{Monotone in $x_2$ target functions.}
    \end{subfigure}
    \caption{Bivariate fitting targets of varying complexity.}
    \label{fig:bivariate_targets}
\end{figure}

When fitting an unconstrained function, a general unconstrained spectral neuron is used. When fitting a monotone in $x_2$ function, both an unconstrained spectral neuron and a monotone in $x_2$ neuron with unconstrained $\mA_0$ and $\mA_1$ and a positive-semidefinite diagonal matrix $\mA_2$ are used. The fitting experiments are conducted using two identical noiseless and noisy protocols to the one described for univariate functions. The results are plotted in \figref{fig:bivariate_scaling_results}. Again, we can see that higher dimensional neurons are able to fit higher-complexity models better. Indeed, a function of complexity $m=13$ is fit quite well by spectral neurons of dimension 15, but the scaling curve of lower-dimensional neurons ``flattens out'' as we add more data. Moreover, for higher-complexity monotone targets are slightly easier to fit with a small amount of data - the inductive bias helps. But as the amount of data increases, unconstrained models and monotone models perform similarly. Of course, the monotone model is monotone by construction, which is important when such monotonicity plays a role in the correctness of a system (e.g., modeling the log-odds of a CDF). Thus, we obtain systems that are both correct by design, and trainable from data.

\begin{figure}[htbp]
    \centering
    \begin{subfigure}[c]{\textwidth}
        \includegraphics[width=\textwidth]{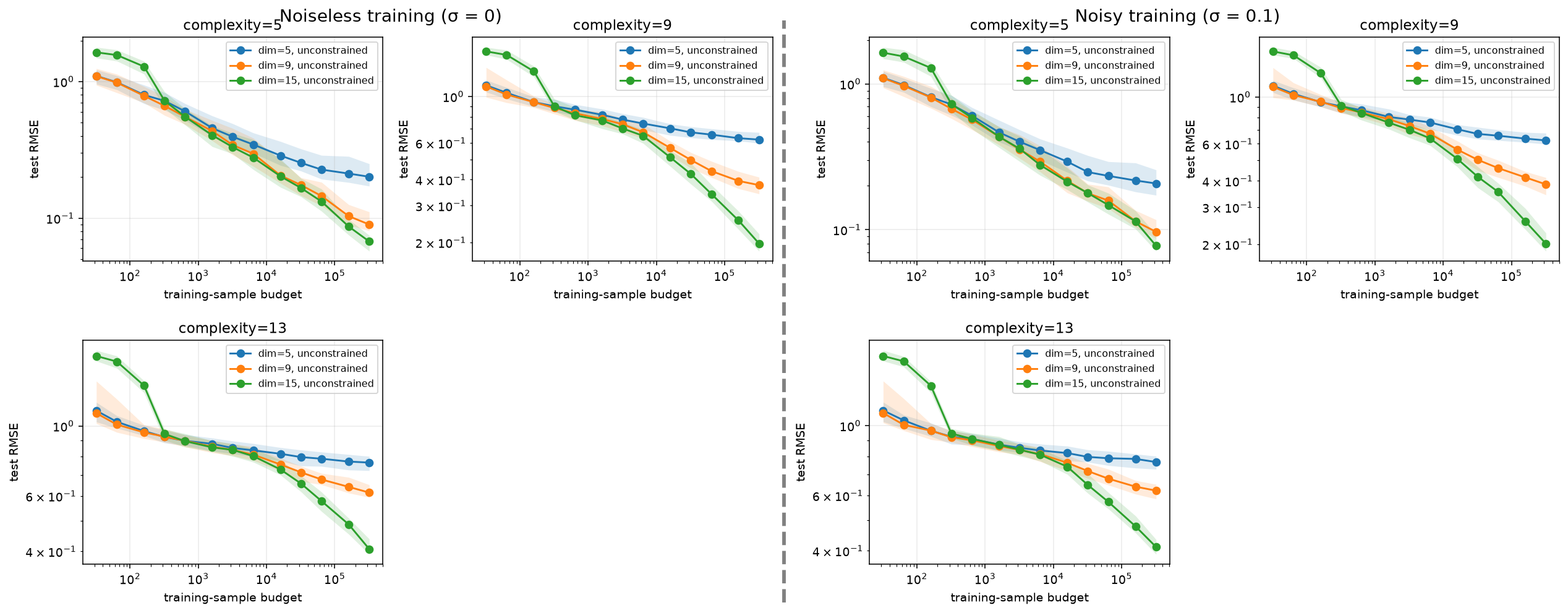}
    \end{subfigure}
    \rule{\textwidth}{0.5pt}
    \begin{subfigure}[c]{\textwidth}
        \includegraphics[width=\textwidth]{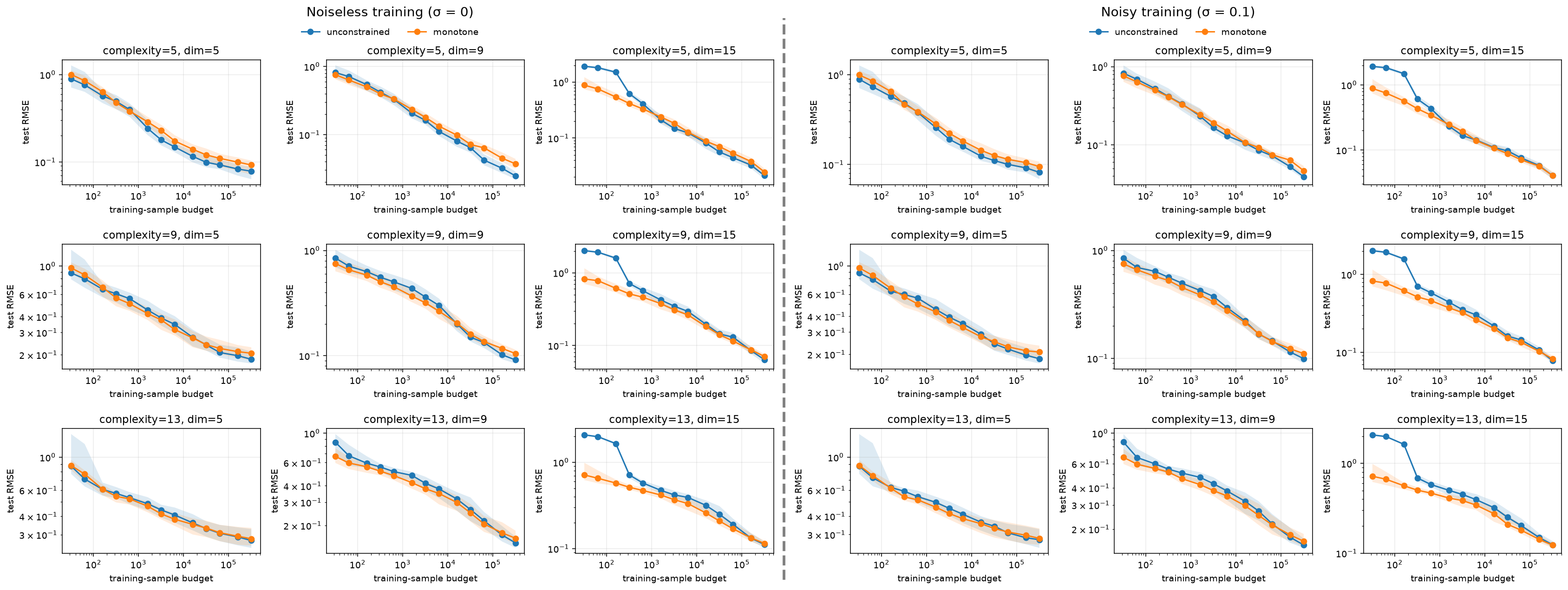}
    \end{subfigure}
    \caption{Bivariate scaling results. $x$-axis - number of training samples. $y$-axis - test loss of the best configuration determined by the validation set. Top - generic targets, grouped by function complexity. Bottom - monotone targets, with complexity along the rows, and matrix dimensions along the columns. Left - noiseless labels. Right - noisy labels.}
    \label{fig:bivariate_scaling_results}
\end{figure}

\subsection{Real datasets}
We perform similar scaling experiments on two real-world datasets that are large enough for such experiments: the Criteo Display Advertising Challenge dataset \citep{criteo}, as a representative of the recommender systems world, and the HIGGS dataset \citep{higgs}, as a representative of the natural sciences. A large dataset allows treating it as a ``distribution'', and sampling mini-batches, just like our synthetic experiment protocol. The sampling method is simple: we treat concatenated randomly shuffled copies of the data-set as an endless stream of samples, from which we take consecutive samples as mini-batches. The dataset sizes are summarized in \tblref{tbl:dataset_sizes}. 

The evaluation protocol is simpler than synthetic experiments. In contrast to synthetic functions, where we could build many random functions of a given complexity, we have only \emph{one} Criteo and \emph{one} HIGGS dataset. Thus, the only sources of randomness are the data-sampling random seeds, and the model initialization seeds.

\begin{table}[htbp]
    \caption{Scaling experiment tabular dataset sizes and column count.}
    \label{tbl:dataset_sizes}
    \centering
    \begin{tabular}{lllll}
        \toprule
        {} & Criteo & Higgs \\
        \midrule
        \# rows & 45,840,617 & 11,000,000 \\
        \# categorical columns & 26 & 0 \\
        \# numerical columns & 13 & 28 \\
        Metric type & Cross-Entropy & Cross-Entropy  \\
        \bottomrule
    \end{tabular}
\end{table}

Both datasets were divided into 80\% training data, 10\% validation data for hyperparameter tuning, and 10\% test data for performance testing of models trained with the tuned hyper-parameters. For the HIGGS dataset, the division into train, validation and test was done by random shuffling. For the Criteo dataset the division was done chronologically, as is customary in recommender systems: we want to train on older samples, but validate and test on newer samples.

HIGGS pre-processing consists of elementary standardization - we compute mean and standard deviation of each numerical column based on the training set. Criteo pre-processing was done in two ways as explained below. The first was was identical to the procedure employed by the winners of the Criteo display advertising challenge - missing values in a categorical column are replaced by a special \texttt{MISSING} categorical value, and rare values are replaced by a special \texttt{RARE} value. Then, the column is one-hot encoded. Numerical columns, which are always integers, are first discretized to form a categorical column via the transformation $x \to 2 + \lfloor\ln^2(x)\rfloor$ if $x \geq 2$, while leaving it intact if $x < 2$. Then, the same procedure as for the categorical columns is applied. When a feature from the validation or test sets was not present in the training set, it is replaced with \texttt{MISSING}.

Thus, our second preprocessing strategy is simple standardization, just like in HIGGS, of the numerical columns transformed by $x \to \ln^2(1 + x)$, a similar transformation to the one employed for the winners' discretization strategy strategy. This aligns with the observation in \citet{shtoff2024functionbasis} that the performance on this data-set improves if we treat its numerical columns as numbers, rather than discrete bins.

For the HIGGS dataset, we test a linear model, and spectral neurons of various dimensions that compete against MLPs with 1, 2, and 3 hidden layers and constant width that is chosen such that the number of parameters is as close as possible to the spectral neuron. As for the Criteo dataset, spectral neurons of various dimensions compete against a linear model, and a factorization machine whose embedding vector dimension is chosen such that the factorization machine has a similar number of parameters to the spectral neuron. The objective is convincing that spectral neurons improve with scaling, and are in the ballpark of their natural competition in terms of performance. Of course, we do \emph{not} claim they achieve SOTA performance - their benefits stem from facilitating shape control and retaining coefficient transparency as the models scale. All models are trained with Adam on up to $2^{28}$ samples for Criteo and $2^{26}$ samples of HIGGS, with 4096 samples per mini-batch, and learning rates tuned for each plotted checkpoint.

The Criteo scaling experiment plots are shown in \figref{fig:criteo_scaling_results}. Also, the factorization machine is employed only on the discrete variant, which is how factorization machines are typically employed in practice. We can observe that indeed treating numbers as numbers, i.e. the ``continuous'' variant, performs better than discretizing them. Moreover, for large enough optimizer step count, linear models ``flatten out'' and stop improving while both spectral neurons and factorization machines improve. Additionally, in contrast to a factorization machine, spectral neurons improve nicely with dimension scaling. Finally, both achieve performance in the same ``ballpark''.

\begin{figure}[htbp]
    \centering
    \begin{subfigure}[c]{\textwidth}
        \includegraphics[width=\textwidth]{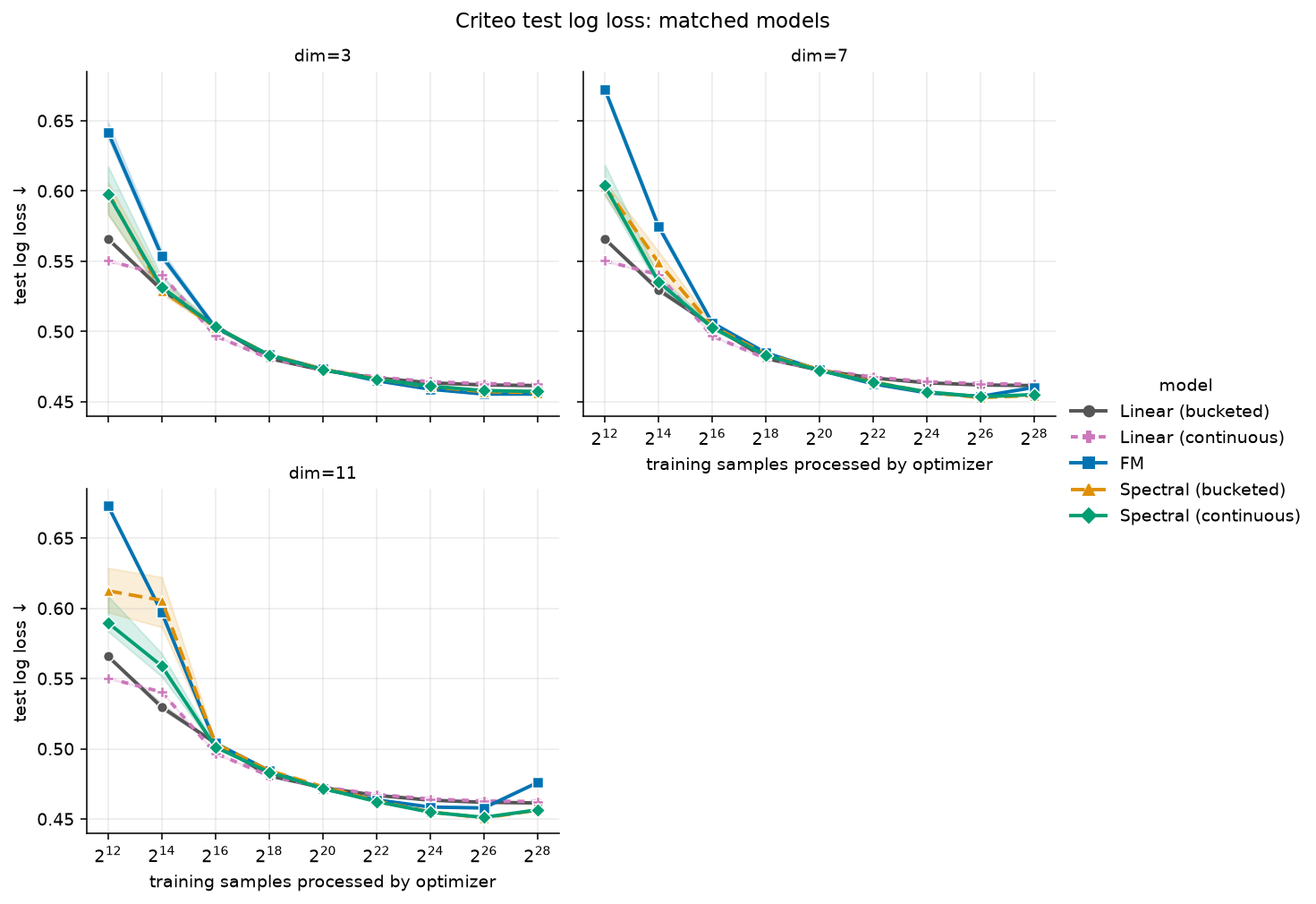}
        \caption{All competing models, grouped by spectral neuron dimension. The embedding dimension of the competing factorization machine (FM) is chosen to match the number of spectral neuron parameters.}
    \end{subfigure} \\
    \vspace{2em}
    \begin{subfigure}[c]{.45\textwidth}
        \includegraphics[width=\textwidth]{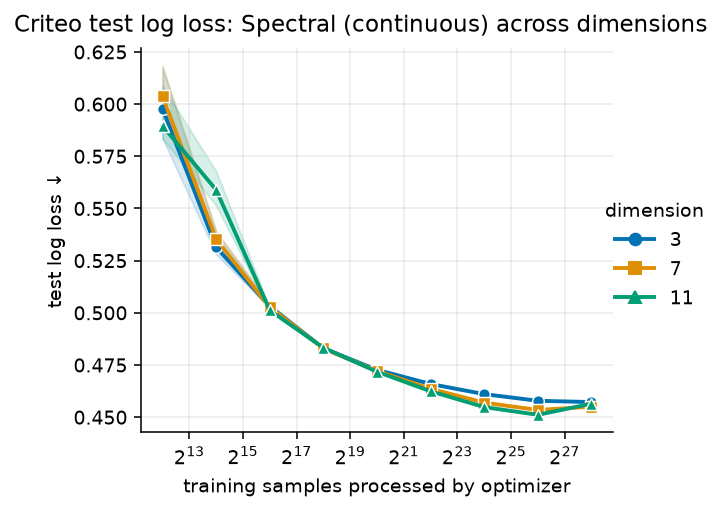}
        \caption{Performance of the spectral neuron for various matrix dimensions with numerical columns treated as numbers.}
    \end{subfigure}
    \hspace{2em}
    \begin{subfigure}[c]{.45\textwidth}
        \includegraphics[width=\textwidth]{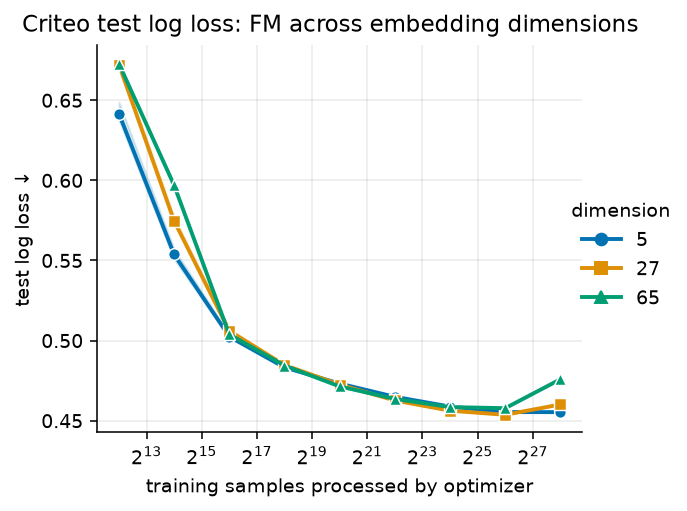}
        \caption{Performance of factorization machines of sizes that match spectral neurons of various matrix dimensions, with numerical columns discretized.}
    \end{subfigure}
    \caption{Results of Criteo scaling experiments.}
    \label{fig:criteo_scaling_results}
\end{figure}

The results for HIGGS scaling experiments are plotted in \figref{fig:higgs_scaling_results}. Similarly, a linear model ``flattens out'' and stops improving after a certain number of samples seen by the optimizer. In contrast, MLPs and spctral neurons do improve. Even though MLPs of depth $\geq 2$ slightly outperform spectral neurons, the performance is in a similar ballpark. Finally, we see that spectral neurons indeed improve with scaling, as claimed.

\begin{figure}[htbp]
    \centering
    \begin{subfigure}[c]{\textwidth}
        \includegraphics[width=\textwidth]{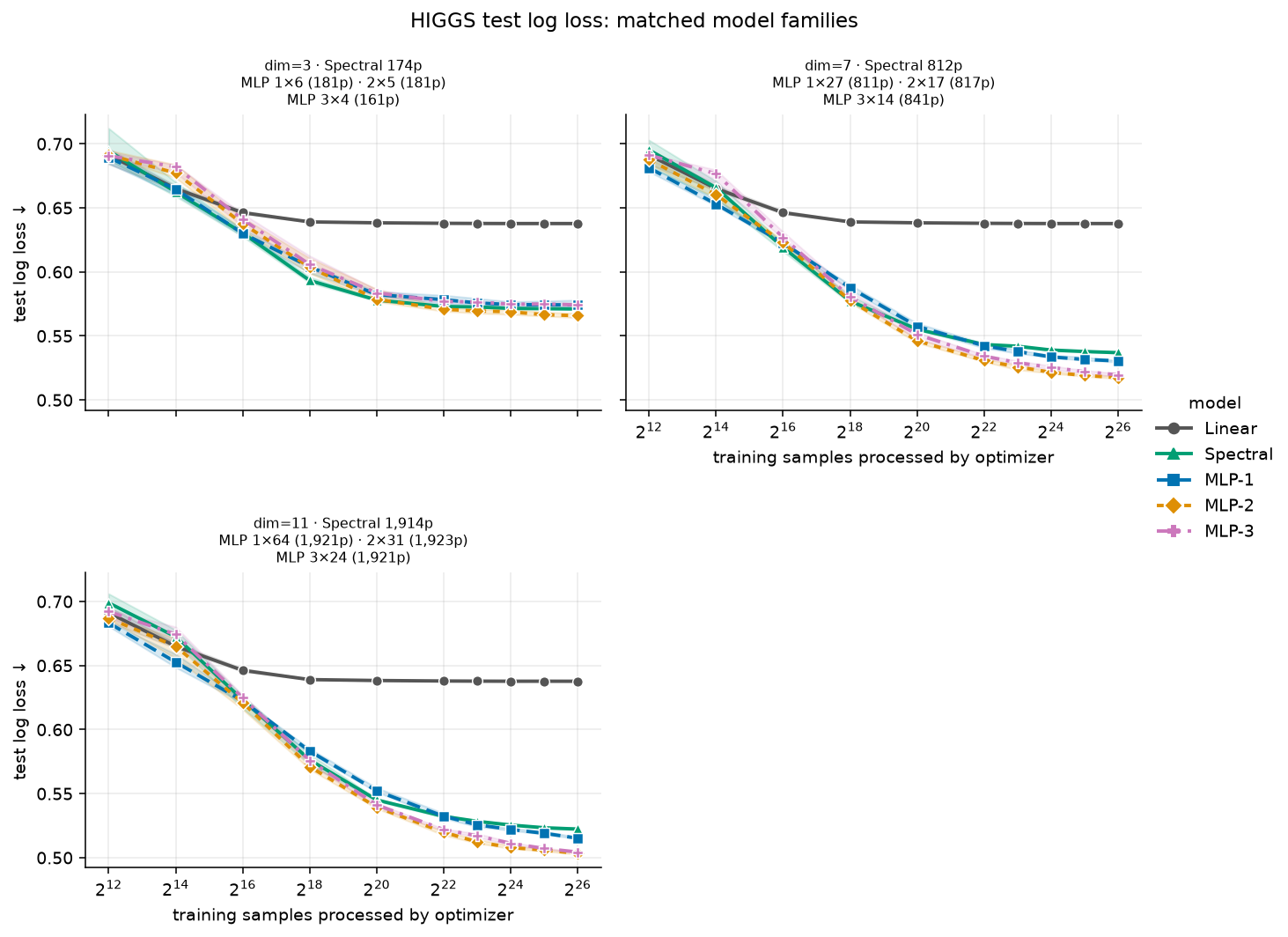}
        \caption{Comparison of models, grouped by spectral neuron dimension. For each MLP depth we show its width and total number of parameters.}
    \end{subfigure}\\
    \vspace{2em}
    \begin{subfigure}[c]{.5\textwidth}
        \includegraphics[width=\textwidth]{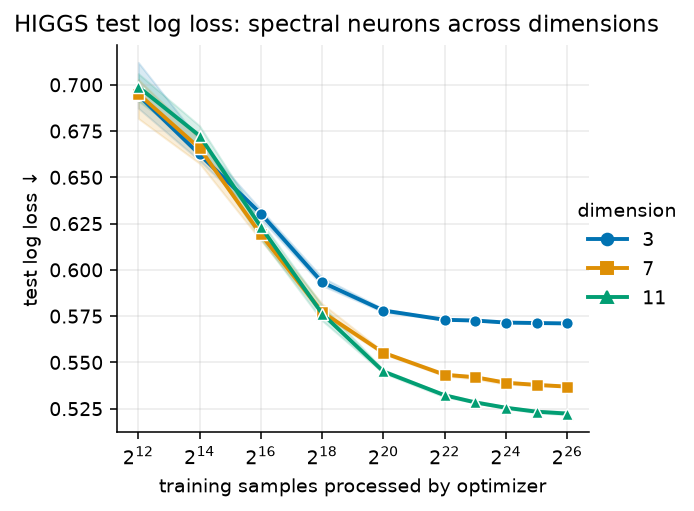}
        \caption{Comparison of spectral neurons of various dimensions.}
    \end{subfigure}
    \caption{Results of HIGGS scaling experiments.}
    \label{fig:higgs_scaling_results}
\end{figure}

\subsection{Global feature-influence bound experiment}
As pointed out in \secref{sec:continuity}, given that a feature $x_i$ changes by $\delta$, the corresponding prediction changes by \emph{at most} $|\delta| \| \mA_i \|_2$. However, this is just an upper bound, and can be vacuous, and thus useless. Testing robustness requires knowing the appropriate model of corruption, so we can estimate a realistic $\delta$, if corruption in one feature is reasonable, and requires domain knowledge. Due to lack of such domain knowledge - we take a simplistic approach and use it for the HIGGS dataset: we select $\epsilon \in [-0.5, 0.5]$ uniformly at random, and corrupt a numerical feature $x_i$ by $\delta = \epsilon \sigma_i$, where $\sigma_i$ is the standard obtained during feature standardization phase described in the pevious section. Then, we measure the actual prediction change and the theoetical upper bound this paper provides:
\[
    \frac{|f(\vx + \delta \ve_i) - f(\vx)|}{|\delta| \| \mA_i \|_2}.
\]
A reasonable corruption model for the Criteo dataset requires more domain knowledge, and since features are anonymized, it is more challenging to devise.

We take trained snapshots of the model with the largest amount of data, meaning the last snapshot, and plot histograms of the above ratio in \figref{fig:higgs_robustness}. At least in this dataset we can see that the bound is not vacuous, even though as matrix dimensons grow, the ratio distribution concentrates more and more around zero. Of course, this is not a general conclusion about any dataset, but this shows that the theoretical bound can be potentially useful on real-world datasets, and thus this aspect of coefficient transparency can be observed in practice. Additionally, we see that the tightness of the upper bound diminishes with dimension, hinting, perhaps, at the need for a spectral norm regularization mechanism that balances model performance with the spectral norms of the matrices $\|\mA_i\|$.

\begin{figure}[htbp]
    \centering
    \includegraphics[height=.9\textheight]{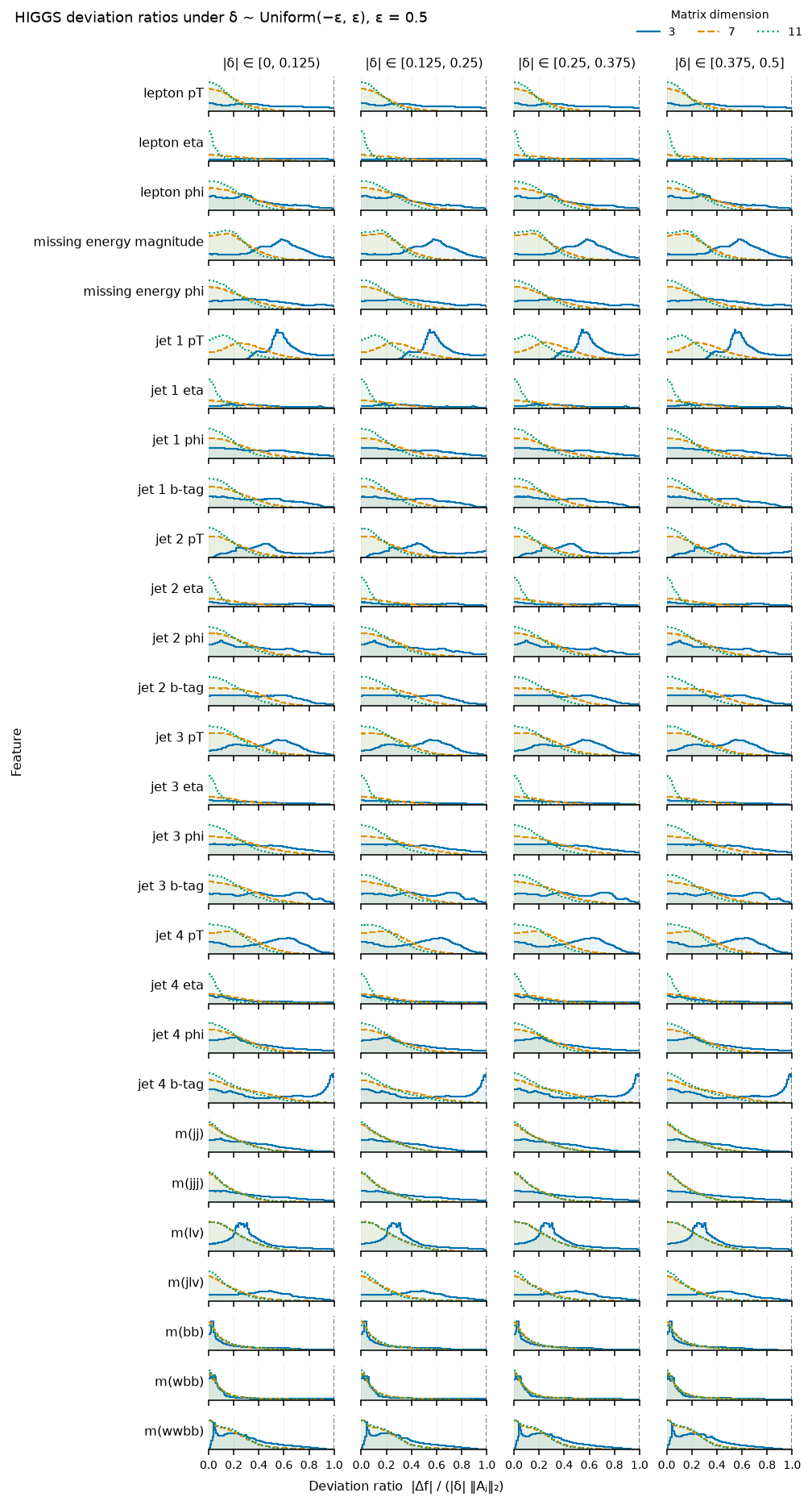}
    \caption{Deviation ratio histograms for all features of the HIGGS dataset across different trained model dimensions.}
    \label{fig:higgs_robustness}
\end{figure}

\section{Discussion and future work}
This paper starts with a quote of a well-known scientist, that perhaps neurons should do more compute than they do now. There are already works on neurons that solve optimization problems, but there is one optimization problem that received decades of research to the point it can be solved extremely reliably - the symmetric eigenvalue problem.

Inspired by the above, in this paper we presented a different notion of an artificial neuron, where the linear combination of features produces a matrix, rather than a scalar, and a non-linearity is an eigenvalue of that matrix. This model retains coefficient transparency analogous to that of linear models, allowing us to reason about feature influence, facilitate shape control by choosing appropriate eigenvalue index and definiteness properties of the learned matrices, while being able to improve with scaling. Such a unique combination is hard to come by.

We have shown just enough to demonstrate that this family may be useful in some machine learning contexts, but there is much more to do. First, computing eigenvalues is an expensive operation - it requires $O(d^3)$ operations for a $d$-dimensional matrix. However, other matrix families may enjoy more efficient computation and parameter count, while losing only a little bit of expressive power. For example, symmetric tri-diagonal matrices that admit efficient eigenvalue algorithms and $O(d)$ parameters \citep{parlett1998symmetric}. 

Multiple learned semi-definite matrices should probably not be chosen to be all diagonal, unless we actively assume they all share the same eigenvector basis. There can be many ways to represent them, such as $L L^T$ for a lower-triangular matrix, or $\sqrt{L L^T}$ where $\sqrt{\cdot}$ denotes the matrix square-root. In this paper we have not tested which parametrization and corresponding initialization strategies perform the best with off-the-shelf optimizers - this is the subject of a future work.

Trainable invertible functions are at the heart of generative flow models. Spline-based normalizing flows construct expressive multivariate bijections from trainable monotone scalar transformations, whose parameters are produced by a conditioner network from the remaining coordinates in a coupling layer, or from preceding coordinates in an autoregressive layer \citep{durkan2019neural,papamakarios2021normalizing}. This suggests replacing the scalar spline function by $g(x) \to \lambda_k(\mA(\vc) + x \mB(\vc))$, where $\vc$ denotes the conditioning coodinates, and $\mB(\vc)$ is positive-definite to ensure the function is invertible. Elementary algebra shows that the inverse itself is available as $g^{-1}(z) = \lambda_{d - k + 1}(\mB^{-1/2}(z \mI - \mA) \mB^{-1/2})$. The study of such a replacement may be another interesting future work. 

Finally, based on our global feature-influence bound experiment, another direction may be studying regularization and feature selection techniques: techniques similar to Lasso to drive weight matrix norms to zero for features with little influence, or techniques to constrain matrix norms to make the model more robust to perturbations.

\section*{AI usage statement}
Except for the related works section, the entirety of this work has been written through the fun activity of conveying knowledge to scientists by writing. In the related works section, AI played a major role in finding related works, and throughout the paper in finding ``canonical'' references to various facts, mainly from linear algebra, widely known by many scientists.

In contrast to the manuscript, the development of the assoicated experimental code has been heavily assisted by OpenAI Codex and reviewer by the authors. Codex usage began after an initial version was manually written, to give Codex a framework of patterns to follow.

\section*{Broad\cite{}er Impact Statement}
We have found advancing science, even if by a small amount, to have only positive impact.

\bibliography{main}
\bibliographystyle{tmlr}

\end{document}